\documentclass{article}
\usepackage[nonatbib,final]{neurips_2021}
\usepackage[utf8]{inputenc} %
\usepackage[T1]{fontenc}    %
\usepackage{url}            %
\usepackage{booktabs}       %
\usepackage{amsfonts}       %
\usepackage{nicefrac}       %
\usepackage{microtype}      %
\usepackage{xcolor}         %
\usepackage{microtype}
\usepackage{graphicx}
\usepackage{amsmath,amsfonts,bm,amssymb,mathtools,amsthm}

\newtheorem{theorem}{Theorem}
\newtheorem{definition}{Definition}

\newtheorem{corollary}{Corollary}
\newtheorem{lemma}{Lemma}
\newtheorem{assumption}{Assumption}
\newtheorem{example}{Example}

\newcommand{\bb}[1]{{\mathbb{#1}}}

\def\eqref#1{equation~\ref{#1}}
\def\1{\bm{1}}

\def\tq{{\overline{q}}}
\newcommand{\defeq}{\vcentcolon=}

\def\vtheta{{\bm{\theta}}}
\def\vpsi{{\bm{\psi}}}

\def\vg{{\bm{g}}}
\def\vh{{\bm{h}}}

\def\vp{{\bm{p}}}
\def\vq{{\bm{q}}}

\def\vv{{\bm{v}}}

\def\vx{{\bm{x}}}
\def\vy{{\bm{y}}}
\def\vz{{\bm{z}}}

\def\vtheta{{\bm{\theta}}}

\def\Et{{E_\vtheta(\vx)}}

\newcommand{\at}[2][]{#1|_{#2}}

\DeclareMathAlphabet{\mathsfit}{\encodingdefault}{\sfdefault}{m}{sl}
\SetMathAlphabet{\mathsfit}{bold}{\encodingdefault}{\sfdefault}{bx}{n}

\def\gH{{\mathcal{H}}}

\def\gL{{\mathcal{L}}}
\def\gM{{\mathcal{M}}}

\def\gO{{\mathcal{O}}}
\def\gP{{\mathcal{P}}}

\def\gX{{\mathcal{X}}}

\newcommand{\KL}{D_{\mathrm{KL}}}

\DeclareMathOperator*{\argmax}{arg\,max}
\DeclareMathOperator*{\argmin}{arg\,min}

\usepackage{algorithmic}
\usepackage{algorithm}
\usepackage{arydshln}
\usepackage{fancyvrb}
\usepackage{enumitem}
\usepackage{multirow}
\usepackage{caption}
\usepackage{subcaption}
\usepackage{wrapfig}
\usepackage{hyperref}

\title{Pseudo-Spherical Contrastive Divergence}

\author{%
  Lantao Yu\\
  Computer Science Department\\
  Stanford University\\
  \texttt{lantaoyu@cs.stanford.edu} \\
  \And 
  Jiaming Song\\
  Computer Science Department\\
  Stanford University\\
  \texttt{tsong@cs.stanford.edu} \\
  \And
  Yang Song\\
  Computer Science Department\\
  Stanford University\\
  \texttt{yangsong@cs.stanford.edu} \\
  \And
  Stefano Ermon\\
  Computer Science Department\\
  Stanford University\\
  \texttt{ermon@cs.stanford.edu} \\
}

\begin{document}

\maketitle

\begin{abstract}
Energy-based models (EBMs) offer flexible distribution parametrization. However, due to the intractable partition function, they are typically trained via contrastive divergence for maximum likelihood estimation. 
In this paper, we propose pseudo-spherical contrastive divergence (PS-CD) to generalize maximum likelihood learning of EBMs. PS-CD is derived from the maximization of a family of strictly proper homogeneous scoring rules, which avoids the computation of the intractable partition function and provides a generalized family of learning objectives that include contrastive divergence as a special case.
Moreover, PS-CD allows us to flexibly choose various learning objectives to train EBMs without additional  computational cost or variational minimax optimization. Theoretical analysis on the proposed method and extensive experiments on both synthetic data and commonly used image datasets demonstrate the effectiveness and modeling flexibility of PS-CD, as well as its robustness to data contamination, thus showing its superiority over maximum likelihood and $f$-EBMs.
\end{abstract}

\section{Introduction}
Energy-based models (EBMs) provide a unified framework for generative and discriminative learning by capturing dependencies between random variables with an energy function. Due to the absence of the normalization constraint, EBMs offer much more flexibility in distribution parametrization and architecture design compared to properly normalized probabilistic models such as autoregressive models \cite{larochelle2011neural,germain2015made}, flow-based models \cite{dinh2014nice,dinh2016density,kingma2018glow}
and sum-product networks \cite{poon2011sum}. 
Recently, deep EBMs have achieved considerable success in realistic image generation \cite{du2019implicit,nijkamp2019learning,dai2019exponential,grathwohl2019your}, molecular modeling \cite{zhang2020deep} and model-based planning \cite{du2019model}, thanks to modern deep neural networks \cite{lecun1998gradient,krizhevsky2012imagenet,he2016deep} for parametrizing expressive energy functions and improved Markov Chain Monte Carlo (MCMC) techniques \cite{neal2011mcmc,robert2013monte,jacob2017unbiased,du2019implicit,nijkamp2019learning} for efficiently sampling from EBMs.

Training EBMs consists of finding an energy function that assigns low energies to correct configurations of variables and high energies to incorrect ones \cite{lecun2006tutorial}, where a central concept is the \emph{loss functional} that is used to measure the quality of the energy function and is minimized during training. 
The flexibility of EBMs does not come for free: it makes the design of loss functionals particularly challenging, as it usually involves the partition function that is generally intractable to compute.
As a result, EBMs are typically trained via CD \cite{hinton2002training}, which belongs to the ``analysis by synthesis'' scheme \cite{grenander2007pattern} and performs a sampling-based estimation of the gradient of KL between data distribution and energy-based distribution. 
Since different loss functionals will induce different solutions in practice (when the model is mis-specified and data is finite) and KL may not provide the right inductive bias \cite{gibbs2002choosing,zhao2018bias}, 
inspired by the great success of implicit generative models \cite{goodfellow2014generative,nowozin2016f,arjovsky2017wasserstein},
\cite{yu2020training} proposed a variational framework to train EBMs by minimizing general
$f$-divergences~\cite{csiszar1964informationstheoretische}. Although this framework enables us to specify various modeling preferences such as diversity/quality tradeoff, they rely on learning additional components (variational functions) within a minimax framework, where the optimization is complicated by the notion of Nash equilibrium and local optimality \cite{jin2019local} and suffers from instability and non-convergence issues \cite{mescheder2018training}. Along this line, noise contrastive estimation (NCE) \cite{gutmann2012noise} can train EBMs with a family of loss functionals induced by different Bregman divergences. 
However, in practice, NCE usually relies on carefully-designed noise distribution such as context-dependent noise distribution \cite{ji2015blackout} or joint learning of a flow-based noise distribution \cite{gao2020flow}.

In this paper, we draw inspiration from statistical decision theory \cite{dawid1998coherent} and propose a novel perspective for designing loss functionals for training EBMs without involving auxiliary models or variational optimization. Specifically, to generalize maximum likelihood training of EBMs, we focus on maximizing pseudo-spherical scoring rules \cite{roby1965belief,good1971comment}, which are \emph{strictly proper} such that the data distribution is the unique optimum and \emph{homogeneous} such that they can be evaluated without the knowledge of the normalization constant.
Under the ``analysis by synthesis'' scheme used in CD and $f$-EBM \cite{yu2020training}, we then derive a practical algorithm termed Pseudo-Spherical Contrastive Divergence (PS-CD). 
Different from $f$-EBM, PS-CD enables us to specify flexible modeling preferences without requiring additional computational cost or unstable minimax optimization.
We provide a theoretical analysis on the sample complexity and convergence property of PS-CD, as well as its connections to maximum likelihood training.
With experiments on both synthetic data and commonly used image datasets, we show that PS-CD achieves significant sample quality improvement over conventional maximum likelihood training and competitive performance to $f$-EBM without expensive variational optimization. Based on a set of recently proposed generative model evaluation metrics \cite{naeem2020reliable}, we further demonstrate the various modeling tradeoffs enabled by PS-CD, justifying its modeling flexibility. Moreover, PS-CD is also much more robust than CD in face of data contamination.

\section{Preliminaries}
\subsection{Energy-Based Distribution Representation and Sampling}
Given a set of \emph{i.i.d.} samples $\{\vx_i\}_{i=1}^N$ from some unknown data distribution $p(\vx)$ defined over the sample space $\gX \subset \bb{R}^m$, the goal of generative modeling is to learn a $\vtheta$-parametrized probability distribution $q_\vtheta(\vx)$ to approximate the data distribution $p(\vx)$. In the context of energy-based modeling, instead of directly parametrizing a properly normalized distribution, we first parametrize an unnormalized energy function $E_\vtheta: \gX \to \bb{R}$, which further defines a normalized probability density via the Boltzmann distribution:
\begin{align}
    q_\vtheta(\vx) = \frac{\overline{q}_\vtheta(\vx)}{Z_\vtheta} = \frac{\exp(- E_\vtheta(\vx))}{Z_\vtheta},
\end{align}
where $Z_\vtheta \defeq \int_\gX \exp(- E_\vtheta(\vx)) \mathrm{d} \vx$ is the partition function (normalization constant). In this paper, unless otherwise stated, we will use $\overline{q}$ to denote an unnormalized density and $q$ to denote the corresponding normalized distribution. We also assume that the exponential of the negative energy belongs to the $L^1$ space, $\mathcal{E} \defeq \left\{E_\vtheta: \gX \to \bb{R} : \int_\gX \exp(- E_\vtheta(\vx)) \mathrm{d} \vx < \infty \right\}$, \emph{i.e.}, $Z_\vtheta$ is finite.

Since energy-based models (EBMs) represent a probability distribution by assigning unnormalized scalar values (energies) to the data points, we can use any model architecture that outputs a bounded real number given an input to implement the energy function, which allows extreme flexibility in distribution parametrization. However, it is non-trivial to sample from an EBM,  usually requiring MCMC \cite{robert2013monte} techniques.
Specifically, in this work we consider using Langevin dynamics \cite{neal2011mcmc,welling2011bayesian},  
a gradient-based MCMC method that performs noisy gradient descent to traverse the energy landscape and arrive at the low-energy configurations:
\begin{align}
    \Tilde{\vx}_t = \Tilde{\vx}_{t-1} - \frac{\epsilon}{2} \nabla_\vx E_\vtheta(\Tilde{\vx}_{t-1}) + \sqrt{\epsilon} \vz_t,\label{eq:langevin}
\end{align}
where $\vz_t \sim \mathcal{N}(0, I)$. The distribution of $\Tilde{\vx}_T$ converges to the model distribution $q_\vtheta(\vx) \propto \exp(- E_\vtheta(\vx))$ when $\epsilon \to 0$ and $T \to \infty$ under some regularity conditions \cite{welling2011bayesian}. 
In order to sample from an energy-based distribution efficiently, many scalable techniques have been proposed such as learning non-convergent, non-persistent, short-run MCMC \cite{nijkamp2019learning} and using a sample replay buffer to improve mixing time and sample diversity \cite{du2019implicit}. In this work, we leverage these recent advances when we need to obtain samples from an EBM.

\subsection{Maximum Likelihood Training of EBMs via Contrastive Divergence}
The predominant approach to training explicit density generative models is to approximately 
minimize the KL divergence between the (empirical) data distribution and model distribution. Minimizing KL divergence is equivalent to the following maximum likelihood estimation (MLE) objective:
\begin{align}
    \min_\vtheta \gL_{\mathrm{MLE}}(\vtheta; p) = \min_\vtheta -\mathbb{E}_{p(\vx)}\left[\log q_\vtheta(\vx)\right]
    = \min_\vtheta \bb{E}_{p(\vx)} [E_\vtheta(\vx)] + \log Z_\vtheta. \label{eq:mle}
\end{align}
Because of the intractable partition function (an integral over the sample space), we cannot directly optimize the above MLE objective. To tackle this issue, \cite{hinton2002training} proposed contrastive divergence (CD) algorithm as a convenient way to estimate the gradient of $\gL_{\mathrm{MLE}}(\vtheta; p)$ using samples from $q_\vtheta$:
\begin{align}
    \nabla_\vtheta \gL_{\mathrm{MLE}}(\vtheta; p) = \bb{E}_{p(\vx)} [\nabla_\vtheta E_\vtheta(\vx)] + \nabla_\vtheta \log Z_\vtheta 
    =  \bb{E}_{p(\vx)}[\nabla_\vtheta E_\vtheta(\vx)] - \bb{E}_{q_\vtheta(\vx)}[\nabla_\vtheta E_\vtheta(\vx)],\label{eq:contrastive-divergence}
\end{align}
which can be interpreted as decreasing the energies of real data from $p$ and increasing the energies of fake data generated by $q_\vtheta$.
As discussed above, evaluating Equation~(\ref{eq:contrastive-divergence}) typically relies on MCMC methods such as the Langevin dynamics sampling procedure defined in Equation~(\ref{eq:langevin}) to produce samples from the model distribution $q_\vtheta$, which induces a surrogate gradient estimation:
\begin{align}
    \nabla_\vtheta \gL_{\mathrm{CD-}K}(\vtheta; p) = \bb{E}_{p(\vx)}[\nabla_\vtheta E_\vtheta(\vx)] - \bb{E}_{q_\vtheta^K(\vx)}[\nabla_\vtheta E_\vtheta(\vx)],
\end{align}
where $q_\vtheta^K$ denotes the distribution after $K$ steps of MCMC transitions from an initial distribution (typically data distribution or uniform distribution), and Equation~(\ref{eq:contrastive-divergence}) corresponds to $\gL_{\mathrm{CD-}\infty}$. 

\subsection{Strictly Proper Scoring Rules}\label{sec:proper-scoring-rules}
Stemming from statistical decision theory \cite{dawid1998coherent}, scoring rules evaluate the quality of probabilistic forecasts by assigning numerical scores based on the predictive distributions and the events that materialize. Formally, consider a compact sample space $\gX$.
Let $\gM$ be a space of all locally 1-integrable non-negative finite measures and $\gP$ be a subspace consisting of all probability measures on the sample space $\gX$.
A scoring rule $S(\vx, q)$ specifies the \emph{utility} of 
forecasting 
using a probability forecast $q \in \gP$ for a given sample $\vx \in \gX$. With slightly abused notation, we write the \emph{expected score} of $S(\vx, q)$ under a reference distribution $p$ as:
\begin{align}
    S(p,q) \defeq \mathbb{E}_{p(\vx)} [S(\vx, q)].
\end{align}
\begin{definition}[Proper Scoring Rules \cite{gneiting2007strictly}]\label{def:proper-scoring-rule}
A scoring rule $S: \gX \times \gP \to \mathbb{R}$ is called proper relative to $\gP$ if the corresponding expected score satisfies:
\begin{align}
    \forall p,q \in \gP, S(p,q) \leq S(p,p).
\end{align}
It is strictly proper if the equality holds if and only if $q=p$.
\end{definition}
In prediction and elicitation problems, strictly proper scoring rules encourage the forecaster to make honest predictions based on their true beliefs \cite{garthwaite2005statistical}. In estimation problems, where we want to approximate a distribution $p$ with another parametric distribution $q_\vtheta$, strictly proper scoring rules provide attractive learning objectives:
\begin{align}
    \argmax_{q_\vtheta \in \gP_\vtheta} S(p,q_\vtheta) = \argmax_{q_\vtheta \in \gP_\vtheta} \mathbb{E}_{p(\vx)}[S(\vx, q_\vtheta)]
    = p \text{~(when $p \in \gP_\vtheta$)}.
\end{align}

When a scoring rule $S$ is strictly proper relative to $\gP$, the associated \emph{generalized entropy function} and \emph{divergence function} are defined as:
\begin{align}
    G(p) \defeq \sup_{q \in \gP} S(p, q) = S(p, p), \quad
    D(p, q) \defeq S(p,p) - S(p, q) \geq 0. \label{eq:entropy-divergence}
\end{align}
$G(p)$ is convex and represents the maximally achievable utility, while $D(p, q)$ is the Bregman divergence \cite{bregman1967relaxation} associated with the convex function $G$ and the equality 
holds only when $p=q$.

Next, we introduce a specific kind of scoring rules that are particularly suitable 
for learning unnormalized statistical models.

\begin{definition}[Homogeneous Scoring Rules \cite{parry2012proper}]\label{def:homogeneous}
A scoring rule is homogeneous if it satisfies (here the domain of the score function is extended to $\gX \times \gM$):
\begin{align}
    \forall \lambda>0, \vx \in \gX,~~S(\vx,q) = S(\vx,\lambda \cdot q).
\end{align}
\end{definition}
Since scaling the model distribution $q$ by a positive constant $\lambda$ does not change the value of a homogeneous scoring rule,
such homogeneity allows us to evaluate it without computing the intractable partition function of an energy-based distribution. 
Thus, strictly proper and homogeneous scoring rules 
are natural candidates for new training objectives of EBMs.

\begin{example}
A notable example of scoring rules is the widely used logarithm score: $S(\vx, q) = \log q(\vx)$. The associated generalized entropy is the negative Shannon entropy: $G(p) = \bb{E}_{p(\vx)}[\log p(\vx)]$, and the associated Bregman divergence is the KL divergence: $D(p, q) = \bb{E}_{p(\vx)}\left[\log (p(\vx)/q(\vx))\right]$. From Definitions~\ref{def:proper-scoring-rule} and \ref{def:homogeneous}, we know that the logarithm score is strictly proper but not homogeneous. 
Specifically, for a $\vtheta$-parametrized energy-based distribution $q_\vtheta = \overline{q}_\vtheta / Z_\vtheta$, since $S(\vx, \overline{q}_\vtheta) = S(\vx, Z_\vtheta \cdot q_\vtheta) = S(\vx, q_\vtheta) + \log Z_\vtheta \neq S(\vx, q_\vtheta)$ and $\log Z_\vtheta$ cannot be ignored during the optimization of $\vtheta$,
we need to use tailored methods such as contrastive divergence \cite{hinton2002training} or doubly dual embedding \cite{dai2018kernel} to tackle the intractable partition function.
\end{example}

\section{Training EBMs by Maximizing Homogeneous Scoring Rules}\label{sec:method}
In this section, we derive a new principle for training EBMs from the perspective of optimizing strictly proper homogeneous scoring rules. All proofs for this section can be found in Appendix~\ref{app:proof-method}.
\subsection{Pseudo-Spherical Scoring Rule}
In this section, we introduce the pseudo-spherical scoring rule, which is a representative family of strictly proper homogeneous scoring rules that have great potentials for training deep energy-based models and allow flexible and convenient specification of modeling preferences, yet have not been explored before in the context of energy-based generative modeling.

\begin{definition}[Pseudo-Spherical Scoring Rule \cite{roby1965belief,good1971comment}]\label{def:spherical} For $\gamma > 0$, the pseudo-spherical scoring rule is defined as:
\begin{align}
    S(\vx,q) \defeq \frac{q(\vx)^{\gamma}}{{(\int_\gX q(\vy)^{\gamma + 1} \mathrm{d} \vy)}^\frac{\gamma}{\gamma + 1}}
    = \frac{\overline{q}(\vx)^{\gamma}}{{(\int_\gX \overline{q}(\vy)^{\gamma + 1} \mathrm{d} \vy)}^\frac{\gamma}{\gamma + 1}}
    = \left(\frac{\overline{q}(\vx)}{\|\overline{q}\|_{\gamma + 1}}\right)^\gamma \label{eq:ps-score}
\end{align}
where $\|\overline{q}\|_{\gamma + 1} \defeq \left(\int_\gX \overline{q}(\vy)^{\gamma + 1} \mathrm{d} \vy\right)^\frac{1}{\gamma + 1}$.

The expected pseudo-spherical score under a reference distribution $p$ is defined as:
\begin{align}
    S_{\text{ps}}(p, q) \defeq \mathbb{E}_{p(\vx)} [S(\vx, q)] = \frac{\mathbb{E}_{p(\vx)} [\overline{q}(\vx)^{\gamma}]}{{(\int_\gX \overline{q}(\vy)^{\gamma + 1} \mathrm{d} \vy)}^\frac{\gamma}{\gamma + 1}}\label{eq:expected-ps-score}
\end{align}
\end{definition}

\begin{example}\label{example:spherical}
The classic spherical scoring rule \cite{friedman1983effective}
is a special case in the pseudo-spherical family, which corresponds to $\gamma = 1$:
\begin{align}
    S(\vx,q) = \frac{\overline{q}(\vx)}{(\int_\gX \tq(\vy)^2 \mathrm{d} \vy)^\frac{1}{2}} = \frac{\tq(\vx)}{\|\tq\|_2}
\end{align}
\end{example}

The family of pseudo-spherical scoring rules is appealing because it introduces a different and principled way for assessing a probability forecast. For example, the spherical scoring rule has an interesting geometric interpretation. Suppose the sample space $\gX$ contains $n$ mutually exclusive and exhaustive outcomes. Then a probability forecast can be represented as a vector $\vq = (q_1, \ldots, q_n)$. Let vector $\vp = (p_1, \ldots, p_n)$ represent the oracle probability forecast. The expected spherical score can be written as:
\begin{align}
    S(p, q) = \bb{E}_{p(\vx)} [S(\vx, q)] = \frac{\sum_i p_i q_i}{\sqrt{\sum_i q_i^2}}
    = \|\vp\|_2 \frac{\langle \vp, \vq \rangle}{\|\vp\|_2 \|\vq\|_2} = \|\vp\|_2 \cos(\angle(\vp, \vq))
\end{align}
where $\langle \vp, \vq \rangle$ and $\angle(\vp, \vq)$ denote the inner product and the angle between vectors $\vp$ and $\vq$ respectively. In other words, when we want to evaluate the expected spherical score of a probability forecast $\vq$ under real data distribution $\vp$ using samples, the angle between $\vp$ and $\vq$ is the sufficient statistics.
Since we know that both $\vp$ and $\vq$ belong to the probability simplex $\gP=\{\vv|\sum_{\vx \in \gX} \vv(\vx) = 1 \text{ and } \forall \vx \in \gX, \vv(\vx) \geq 0.\}$, the expected score will be minimized if and only if the angle is zero, which implies $\vp=\vq$. More importantly, 
since all we need to do is to minimize the angle of the deviation, we are allowed to scale $\vq$ by a constant. Specifically, when $\vq$ is an energy-based distribution 
$\vq = \left(\frac{\exp(-E_1)}{\sum_i \exp(-E_i)}, \ldots, \frac{\exp(-E_n)}{\sum_i \exp(-E_i)}\right)$, 
we can instead evaluate and minimize the angle between data distribution $\vp$ and the unnormalized distribution $\overline{\vq} = \left(\exp(-E_1), \ldots, \exp(-E_n)\right)$, since $\angle(\vp, \vq) = \angle(\vp, \overline{\vq})$.
More generally, we have the following theorem to justify the use of pseudo-spherical scoring rules for training energy-based models:
\begin{theorem}[\cite{gneiting2007strictly,parry2016linear}]\label{the:pseudo-spherical}
Pseudo-spherical scoring rule is strictly proper and homogeneous.
\end{theorem}
As the original definition of pseudo-spherical scoring rule (Equation~(\ref{eq:ps-score})) takes the form of a fraction, for computational considerations, in this paper we instead focus on optimizing its \emph{composite scoring rule} (Definition 2.1 in \cite{kanamori2014affine}):

\begin{definition}[$\gamma$-score \cite{fujisawa2008robust}]
For the expected pseudo-spherical score $S_{ps}(p,q)$ defined in Equation~(\ref{eq:expected-ps-score}) with $\gamma > 0$, the expected $\gamma$-score is defined as:
\begin{align}
S_\gamma(p,q) \defeq & \frac{1}{\gamma} \log (S_\text{ps}(p,q))
= \frac{1}{\gamma}  \log \left(\mathbb{E}_{p(\vx)} [\tq(\vx)^\gamma]\right) - \log(\|\tq\|_{\gamma + 1})\label{eq:gamma-score}
\end{align}
\end{definition}
Since $\frac{1}{\gamma} \log (u)$ is strictly increasing in $u$, $S_\gamma(p,q)$ is a strictly proper homogeneous composite score:
\begin{align}
    \argmax_{q \in \gP} S_\gamma(p,q) = \argmax_{q \in \gP} \frac{1}{\gamma} \log (S_{ps}(p,q))
    = \argmax_{q \in \gP} S_{ps}(p,q) = p.
\end{align}

\subsection{Pseudo-Spherical Contrastive Divergence}
Suppose we parametrize the energy-based model distribution as $q_\vtheta \propto \overline{q}_\vtheta = \exp(- E_\vtheta)$ and we want to minimize the negative $\gamma$-score in Equation~(\ref{eq:gamma-score}):
\begin{align}
    \min_\vtheta \gL_\gamma(\vtheta; p) = \min_\vtheta -\frac{1}{\gamma}  \log \left(\mathbb{E}_{p(\vx)} [\tq_\vtheta(\vx)^\gamma]\right)
    +  \log(\|\tq_\vtheta\|_{\gamma + 1})\label{eq:loss-gamma}
\end{align}
In the following theorem, we derive the gradient of $\gL_\gamma(\vtheta; p)$ with respect to $\vtheta$:
\begin{theorem}\label{the:ps-cd-gradient}
The gradient of $\gL_\gamma(\vtheta; p)$ with respect to $\vtheta$ can be written as:
\begin{align}
    \nabla_\vtheta \gL_\gamma(\vtheta; p) = -\frac{1}{\gamma} \nabla_\vtheta \log \left(\mathbb{E}_{p(\vx)} [\exp(- \gamma E_\vtheta(\vx))]\right) 
    - \mathbb{E}_{r_{\vtheta}(\vx)}[\nabla_\vtheta E_\vtheta(\vx)]\label{eq:ps-cd-gradient}
\end{align}
where the auxiliary distribution $r_\vtheta$ is also an energy-based distribution defined as: $$r_\vtheta(\vx) \defeq \frac{\tq_\vtheta(\vx)^{\gamma + 1}}{\int_\gX \tq_\vtheta(\vx)^{\gamma + 1} \mathrm{d} \vx} =  \frac{\exp(-(\gamma + 1) E_\vtheta(\vx))}{\int_\gX \exp(-(\gamma + 1) E_\vtheta(\vx)) \mathrm{d} \vx}.$$
\end{theorem}
In App.~\ref{app:proof-ps-cd-gradient}, we provide two different ways to prove the above theorem. The first one is more straightforward and directly differentiates through the term $\log (\|\tq_\vtheta\|_{\gamma+1})$. The second one leverages a variational representation of $\log (\|\tq_\vtheta\|_{\gamma+1})$, where the optimal variational distribution happens to take an analytical form of $r^*_\vtheta(\vx) \propto \tq_\vtheta(\vx)^{\gamma + 1}$, thus avoiding the minimax optimization in other variational frameworks such as \cite{yu2020training,dai2018kernel,dai2019exponential}.
The main challenge in maximizing $\gamma$-score is that it is generally intractable to exactly compute the gradient of the second term in Equation~(\ref{eq:gamma-score}). 

During training, estimating the second term of Equation~(\ref{eq:ps-cd-gradient}) requires us to obtain samples from the auxiliary distribution $r_{\vtheta} \propto \exp(-(\gamma+1)E_\vtheta)$, while at test time, we want to sample from the model distribution $q_\vtheta \propto \exp(-E_\vtheta)$ that approximates the data distribution. Due to the restrict regularity conditions on the convergence of Langevin dynamics,
in practice, we found it challenging to use the iterative sampling process in Equation~(\ref{eq:langevin}) with a fixed number of transition steps and step size to produce samples from $r_{\vtheta}$ and $q_\vtheta$ simultaneously, as the temperature $\gamma + 1$ in $r_{\vtheta}$ simply amounts to a linear rescaling to the energy function during training.
Thus for generality, as in contrastive divergence \cite{hinton2002training,du2019implicit,nijkamp2019learning} and $f$-EBM \cite{yu2020training}, we make the minimal assumption that we only have a sampling procedure to produce samples from $q_\vtheta$ for both learning and inference procedures.

In this case, we can leverage the analytical form of $r_{\vtheta}$ and \emph{self-normalized importance sampling} \cite{mcbook} (which has been used to derive gradient estimators in other contexts such as importance weighted autoencoder \cite{burda2015importance,finke2019importance}) to obtain a consistent estimation of Equation~(\ref{eq:ps-cd-gradient}):
\begin{theorem}\label{the:importance-sampling}
Let $\vx_1^+, \ldots, \vx_N^+$ be i.i.d. samples from $p(\vx)$ and $\vx_1^-, \ldots, \vx_N^-$ be i.i.d. samples from $q_\vtheta(\vx) \propto \exp(-E_\vtheta(\vx))$. Define the gradient estimator as:
\begin{align}
    \widehat{\nabla_\vtheta \gL_\gamma^N(\vtheta;p)} \defeq 
     - \nabla_\vtheta \frac{1}{\gamma} \log \left( \frac{1}{N} \sum_{i=1}^N \exp(- \gamma E_\vtheta(\vx_i^+))\right)
     - \frac{\sum_{i=1}^N \omega_\vtheta(\vx_i^-) \nabla_\vtheta E_\vtheta(\vx_i^-)}{\sum_{i=1}^N \omega_\vtheta(\vx_i^-)}\label{eq:gradient-estimator}
\end{align}
where the self-normalized importance weight $\omega_\vtheta(\vx_i^-) \defeq \overline{r}_{\vtheta}(\vx_i^-) / \overline{q}_\vtheta (\vx_i^-) = \exp(- \gamma E_\vtheta(\vx_i^-))$.
Then the gradient estimator converges to the true gradient (Equation~(\ref{eq:ps-cd-gradient})) in probability, i.e., $\forall \epsilon > 0$:
$$\lim_{N \to \infty} \bb{P}\left(\left\|\widehat{\nabla_\vtheta \gL_\gamma^N(\vtheta;p)} - \nabla_\vtheta \gL_\gamma(\vtheta;p)\right\| \geq \epsilon \right) =0.$$
\end{theorem}
We summarize the pseudo-spherical contrastive divergence (PS-CD) training procedure in Algorithm~\ref{alg:ps-cd}. In Appendix~\ref{sec:app-alg}, we also provide a simple PyTorch implementation for stochastic gradient descent (SGD) with the gradient estimator in Equation~(\ref{eq:gradient-estimator}).

\begin{algorithm}[t]
   \caption{Pseudo-Spherical Contrastive Divergence.}
   \label{alg:ps-cd}
\begin{algorithmic}[1]
   \STATE {\bfseries Input:} Empirical data distribution $p_\text{data}$. Pseudo-spherical scoring rule hyperparameter $\gamma$.
   \STATE Initialize energy function $E_\vtheta$.
   \REPEAT
   \STATE Draw a minibatch of samples $\{\vx_1^+, \ldots, \vx_N^+\}$ from $p_\text{data}$.
   \STATE Draw a minibatch of samples $\{\vx_1^-, \ldots, \vx_N^-\}$ from $q_\vtheta \propto \exp(- E_\vtheta)$ (\emph{e.g.}, using Langevin dynamics with a sample replay buffer).
   \STATE Update the energy function by stochastic gradient descent:
   {\small
   \begin{align*}
       \widehat{\nabla_\vtheta \gL_\gamma^N(\vtheta;p)} = 
    - \nabla_\vtheta \frac{1}{\gamma} \log \left( \frac{1}{N} \sum_{i=1}^N \exp(- \gamma E_\vtheta(\vx_i^+))\right)
    - \frac{\sum_{i=1}^N \exp(- \gamma E_\vtheta(\vx_i^-)) \nabla_\vtheta E_\vtheta(\vx_i^-)}{\sum_{i=1}^N \exp(- \gamma E_\vtheta(\vx_i^-))}
   \end{align*}}
   \UNTIL{Convergence}
\end{algorithmic}
\end{algorithm}

\subsection{Connections to Maximum Likelihood Estimation and Extension to $\gamma<0$}
From Equation~(\ref{eq:entropy-divergence}) in Section~\ref{sec:proper-scoring-rules}, we know that
$\gamma$-score induces the following Bregman divergence (the divergence function associated with proper composite scoring rule is analogously defined in Def. 2.1 in \cite{kanamori2014affine}):
\begin{align*}
    D_\gamma(p,q_\vtheta) = S_\gamma(p,p) - S_\gamma(p,q_\vtheta)
\end{align*}
and maximizing $\gamma$-score is equivalent to minimizing $D_\gamma(p,q_\vtheta)$. In the following lemma, we show that when $\gamma \to 0$, $D_\gamma(p,q_\vtheta)$ will recover the KL divergence between $p$ and $q_\vtheta$, and the gradient of PS-CD will recover the gradient of CD.
\begin{lemma}
When $\gamma \to 0$, we have:
\begin{align*}
    \lim_{\gamma \to 0} D_\gamma(p,q_\vtheta) = \KL(p\|q_\vtheta);\quad
    \lim_{\gamma \to 0} \nabla_\vtheta \gL_\gamma(\vtheta;p) = \nabla_\vtheta \gL_{\mathrm{MLE}}(\vtheta; p).
\end{align*}
\end{lemma}

Inspired by \cite{van2014renyi,li2016renyi} that generalize R\'enyi divergence beyond its definition to negative orders, we now consider the extension of $\gamma$-score with $\gamma < 0$ (although it may not be
strictly proper for these $\gamma$ values). The following lemma shows that maximizing such scoring rule amounts to maximizing a lower bound of logarithm score (MLE) with an additional R\'enyi entropy regularization.

\begin{lemma}
When $-1 \leq \gamma < 0$, we have:
\begin{align*}
    S_\gamma(p, q) \leq \bb{E}_{p(\vx)} [\log q(\vx)] + \frac{\gamma
    }{\gamma + 1} \gH_{\gamma+1}(q)
\end{align*}
where $\gH_{\gamma + 1} (q)$ is the R\'enyi entropy of order $\gamma + 1$.
\end{lemma}

\section{Theoretical Analysis}
In this section, to gain a deeper understanding of our PS-CD algorithm and how the proposed estimator behaves, we analyze its sample complexity and convergence property under certain conditions.
All the proofs for this section can be found in Appendix~\ref{app:theoretical-analysis}.

\subsection{Sample Complexity}
We start with analyzing the sample complexity of the consistent gradient estimator in Equation~(\ref{eq:gradient-estimator}), that is how fast it approaches the true gradient value.
We first make the following assumption:
\begin{assumption}\label{assump:energy}
The energy function is bounded by $K$ and the gradient is bounded by $L$:
$$\forall \vx \in \gX, ~\vtheta \in \Theta, ~|E_\vtheta(\vx)| \leq K, ~\|\nabla_\vtheta E_\vtheta(\vx)\| \leq L.$$
\end{assumption}
This assumption is typically easy to satisfy because in practice we always use a bounded sample space (\emph{e.g.} normalizing images to [0,1] or truncated Gaussian) to ensure stability. For example, in image modeling experiments, we use $L_2$ regularization on the outputs of the energy function (hence bounded energy values), as well as normalized inputs and spectral normalization \cite{miyato2018spectral} for the neural network that realizes the energy function (hence Lipschitz continuous with bounded gradient).

With vector Bernstein inequality \cite{kohler2017sub,gross2011recovering}, we have the following theorem showing a sample complexity of $O\left(\frac{\log(1/\delta)}{\epsilon^2}\right)$ such that the estimation error is less than $\epsilon$ with probability at least $1-\delta$:
\begin{theorem} \label{the:sample-complexity}
For any constants $\epsilon > 0$ and $\delta \in (0,1)$, when the number of samples $N$ satisfies:
\begin{align*}
    N \geq \frac{32 L^2 e^{8 \gamma K}\left(1 + 4\log(2/\delta)\right)}{\epsilon^2}
\end{align*}
we have:
\begin{align*}
    \bb{P}\left(\left\|\widehat{\nabla_\vtheta \gL_\gamma^N(\vtheta;p)} - \nabla_\vtheta \gL_\gamma(\vtheta;p)\right\| \leq \epsilon \right) \geq 1 - \delta.
\end{align*}
\end{theorem}

\subsection{Convergence}
Typically, convergence of SGD are analyzed for unbiased gradient estimators, while the gradient estimator in PS-CD is asymptotically consistent but biased.
Building on the sample complexity bound above and prior theoretical works for analyzing SGD \cite{chen2018stochastic,ghadimi2013stochastic}, we analyze the convergence of PS-CD.
For notational convenience, we use $\gL(\vtheta)$ to denote the loss function $\gL_\gamma(\vtheta;p) = - S_\gamma(p, q_\vtheta)$. Besides Assumption~\ref{assump:energy}, we further make the following assumption on the smoothness of $\gL(\vtheta)$:
\begin{assumption}\label{assump:smooth}
The loss function $\gL(\vtheta)$ is $M$-smooth (with $M>0$):
$$\forall \vtheta_1, \vtheta_2 \in \Theta, ~\|\nabla \gL(\vtheta_1) - \nabla \gL(\vtheta_2)\| \leq M \|\vtheta_1 - \vtheta_2\|.$$
\end{assumption}
This is a common assumption for analyzing first-order optimization methods, which is also used in \cite{ghadimi2013stochastic,chen2018stochastic}. Also this is a relatively mild assumption since we do not require the loss function to be convex in $\vtheta$. Since in non-convex optimization, the convergence criterion is typically measured by gradient norm, following \cite{nesterov2013introductory,ghadimi2013stochastic}, we use $\|\nabla \gL(\vtheta)\| \leq \xi$ to judge whether a solution $\vtheta$ is approximately a stationary point.

\begin{theorem}\label{the:sgd-convergence}
For any constants $\alpha \in (0,1)$ and $\delta \in (0,1)$, suppose that the step sizes satisfy $\eta_t < 2(1-\alpha)/M$ and the sample size $N_t$ used for estimating $\widehat{\vg}_t$ is sufficiently large (satisfying Equation~(\ref{eq:app-convergence-sample})).
Let $\gL^*$ denote the minimum value of $\gL(\vtheta)$. Then with probability at least $1-\delta$, the output of Algorithm~\ref{alg:randomized-sgd} (in Appendix~\ref{app:ps-cd-convergence}), $\widehat{\vtheta}$, satisfies (constant $C \defeq \alpha M L^2 e^{4 \gamma K}$):
\begin{align*}
    \bb{E}[\|\nabla \gL(\widehat{\vtheta})\|^2] < \frac{2(\gL(\vtheta_1) - \gL^*) + 12 C \sum_{t=1}^T \eta_t^2}{\sum_{t=1}^T (2(1-\alpha)\eta_t - M \eta_t^2)}
\end{align*}
\end{theorem}
The above theorem implies the following corollary, which shows a typical convergence rate of $O(1/\sqrt{T})$ for non-convex optimization problems:
\begin{corollary}\label{cor:constant-step-size}
Under the conditions in Theorem~\ref{the:sgd-convergence} except that we use constant step sizes:
$\eta_t = \min \{(1-\alpha)/M, 1/\sqrt{T}\}$ for $t=1,\ldots,T$. Then with probability at least $1-\delta$, we have (constant $C \defeq \alpha M L^2 e^{4 \gamma K}$):
{\small
\begin{align*}
    \bb{E}[\|\nabla \gL(\widehat{\vtheta})\|^2] < \frac{2M(\gL(\vtheta_1) - \gL^*)}{(1-\alpha)^2 T} + \frac{2(\gL(\vtheta_1) - \gL^*) + 12C}{(1-\alpha)\sqrt{T}}
\end{align*}
}
\end{corollary}
In Appendix~\ref{app:ps-cd-convergence}, we discuss more on the strongly convex (Theorem~\ref{the:app-sgd-convergence-strong-convex}) and convex cases (Theorem~\ref{the:app-sgd-convergence-convex}).

\section{Related Work}
\textbf{Direct KL Minimization.} Under the ``analysis by synthesis'' scheme, \cite{hinton2002training} proposed Contrastive Divergence (CD), which estimates the gradient of the log-partition function (arising from KL) using samples from some MCMC procedure. To improve the mixing time of MCMC, \cite{du2019implicit} proposed to employ Persistent CD and a replay buffer to store intermediate samples from Markov chains throughout training, and \cite{nijkamp2019learning} proposed to learn non-convergent short-run MCMC. Both approaches (long-run and short-run MCMC) work well with PS-CD in our experiments. PS-CD may also benefit from recent advances on unbiased MCMC \cite{jacob2017unbiased,qiuunbiased}, which we leave as interesting future work.

\textbf{Fenchel Duality.} By exploiting the primal-dual view of KL, recent works \cite{dai2018kernel,dai2019exponential,arbel2020kale} proposed to cast maximum likelihood training of EBMs as minimax problems, which introduce a dual sampler and are approximately solved by alternating gradient descent ascent updates. Along this line, to allow flexible modeling preferences, \cite{yu2020training} proposed $f$-EBM to enable the use of any $f$-divergence to train EBMs, which also relies on minimax optimization. By contrast, in this work, we leverage the analytical form of the optimal variational distribution and self-normalized importance sampling to reach a framework that requires no adversarial training and has no additional computational cost compared to CD while allowing flexible modeling preferences. Besides convenient optimization, PS-CD and $f$-EBM trains EBMs with two different families of divergences (hence complementary) with KL being the only shared one, since any pseudo-spherical scoring rule corresponds to a Bregman divergence (Section~\ref{sec:proper-scoring-rules}) and the only member in $f$-divergence that is also Bregman divergence is $\alpha$-divergence (with KL as special case) (Theorem~4 in \cite{amari2009divergence}).

\textbf{Homogeneous Scoring Rules.} \cite{takenouchi2015empirical} proposed to learn unnormalized statistical models on discrete sample space by maximizing $\gamma$-score, which uses empirical data distribution ($\widehat{p}(\vx) = n_{\vx}/n$, where $n_{\vx}$ is the number of appearance of $\vx$ in the dataset and $n$ is the total number of data) as a surrogate to the real data density $p$ and relies on a localization trick to bypass the computation of $\|q_\vtheta\|_{\gamma + 1}$.
Consequently, it is only amenable to finite discrete sample space such as natural language \cite{labeau2019experimenting}, whereas PS-CD is applicable to any unnormalized probabilistic model in continuous domains. Another popular homogeneous scoring rule is the Hyv\"arinen score, which gives rise to the score matching objective \cite{hyvarinen2005estimation} for EBM training. However, score matching and its variants~\cite{vincent2011connection,song2019sliced} have difficulties in low data density regions and do not perform well in practice when training EBMs on high-dimensional datasets~\cite{song2019generative}. Moreover, since the score matching objective involves the Hessian of log-density functions that is generally expensive to compute \cite{martens2012estimating}, methods such as approximate propagation \cite{kingma2010regularized}, curvature propagation \cite{martens2012estimating} and sliced score matching \cite{song2020sliced} are needed to approximately compute the trace of the Hessian.

\textbf{Noise Contrastive Estimation.} Another principle for learning EBMs is Noise Contrastive Estimation~(NCE) \cite{gutmann2012noise}, where an EBM is learned by contrasting a prescribed noise distribution with tractable density against the unknown data distribution. Using various Bregman divergences, NCE can be generalized to a family of different loss functionals~\cite{gutmann2011bregman,uehara2020unified}. However, finding an appropriate noise distribution for NCE is highly non-trivial. 
In practice, NCE typically works well in conjunction with a carefully-designed noise distribution such as context-dependent noise distribution \cite{ji2015blackout} or joint learning of a flow-based noise distribution \cite{gao2020flow}. 

In this work, we focus on generalizing maximum likelihood by deriving novel training objectives for EBMs without involving auxiliary models (\emph{e.g.}, the variational function in \cite{yu2020training}, the flow-based noise distribution in \cite{gao2020flow} and the amortized sampler in \cite{kumar2019maximum,dai2019exponential,grathwohl2020no}).

\section{Experiments} \label{sec:experiments}
In this section, we demonstrate the effectiveness of PS-CD on several 1-D and 2-D synthetic datasets as well as commonly used image datasets.

\textbf{Setup.} The 2-D synthetic datasets include Cosine, Swiss Roll, Moon, Mixture of Gaussian, Funnel and Rings, which cover different modalities and geometries (see Figure~\ref{fig:hist2d} in App.~\ref{app:synthetic-data} for illustration). To test the practical usefulness, we use MNIST \cite{lecun1998gradient}, CIFAR-10 \cite{krizhevsky2009learning} and CelebA \cite{liu2015faceattributes} in our experiments for modeling natural images. Following \cite{song2019generative}, for CelebA, we first center-crop the images to $140 \times 140$, then resize them to $64 \times 64$. More experimental details about the data processing, model architectures, sampling strategies and additional experimental results can be found in App.~\ref{app:experiments}.

\begin{wrapfigure}{r}{6cm}
    \centering
    \includegraphics[width=.28\textwidth]{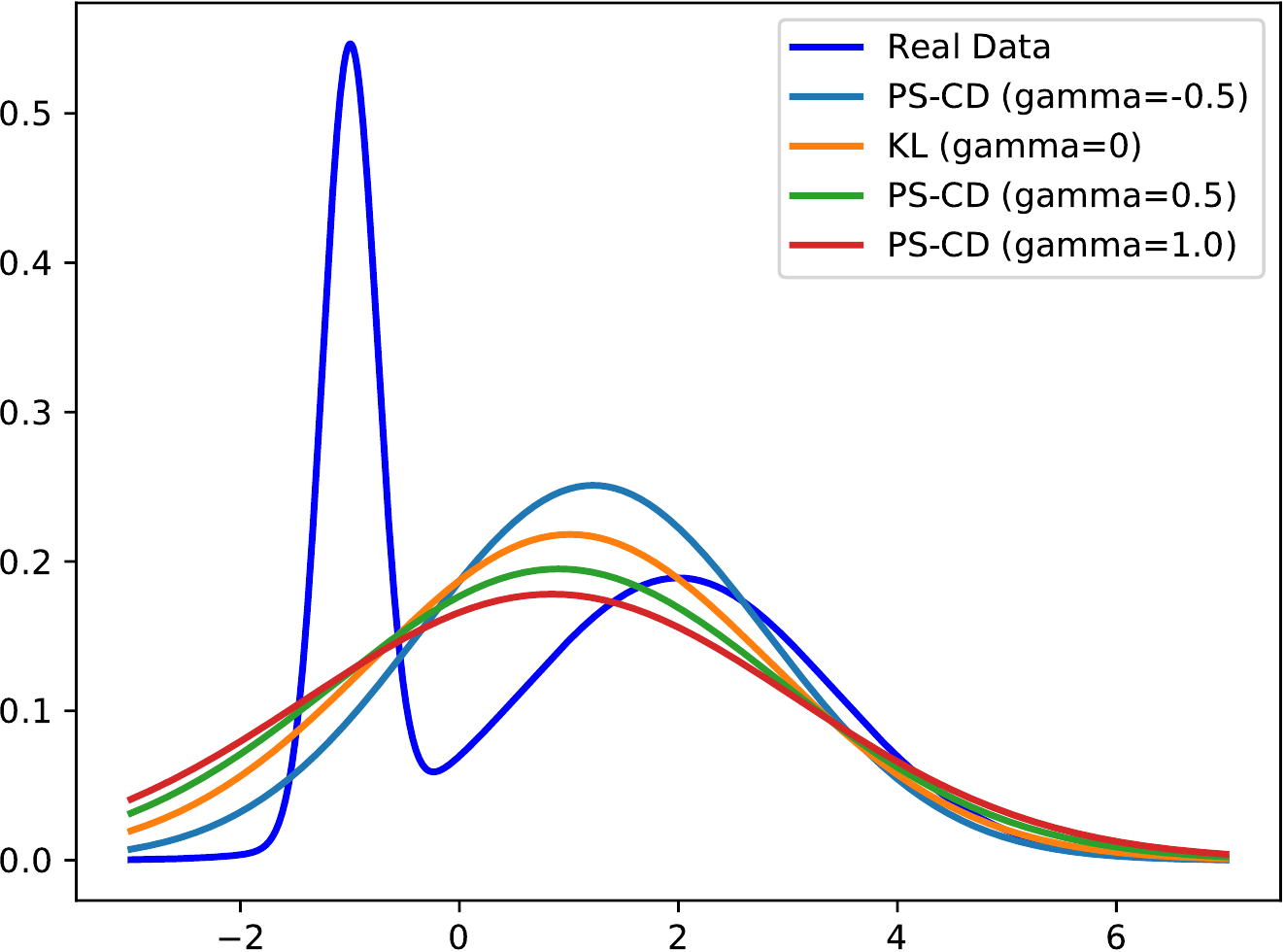}
    \caption{The effects of different $\gamma$ values when fitting a mixture of Gaussian with a single Gaussian.}
    \label{fig:1d-synthetic}
    \vspace{-10pt}
\end{wrapfigure}

\textbf{Effects of Different $\gamma$ Values.} 
To illustrate the modeling flexibility brought by PS-CD and provide insights on the effects of different $\gamma$ values, we first conduct a $1$-D synthetic experiment similar to the one in \cite{yu2020training}. As shown in Figure~\ref{fig:1d-synthetic}, when fitting a mixture of Gaussian with a single Gaussian (\emph{i.e.}, model mis-specification case), the family of PS-CD offers flexible tradeoffs between quality and diversity (\emph{i.e.}, mode collapse vs. mode coverage). Although in the well-specified case these objectives induce the same optimal solution, in this example, we can see that a larger $\gamma$ leads to higher entropy. 
More importantly, compared to $f$-EBM \cite{yu2020training} that also provides similar modeling flexibility and includes CD as a special case, PS-CD does not require expensive and unstable minimax optimization (no additional computational cost compared to CD).
In App.~\ref{app:visualize-objective-landscape}, we further visualize the objective landscapes for different $\gamma$ values. As shown in Figure~\ref{fig:well-specified} and \ref{fig:mis-specified}, when the model is well-specified, different objectives will induce the same optimal solution since they are strictly proper; when the model is mis-specified (corresponding to practical scenarios), different objectives will exhibit different modeling preferences.

Furthermore, to better demonstrate the modeling flexibility brought by PS-CD in high-dimensional case, we conduct experiments on CIFAR-10 using a set of indicative and reliable metrics (Density, Coverage, Precision, Recall) proposed by \cite{naeem2020reliable}  to evaluate the effects of $\gamma$ from various perspectives. Please refer to 
App.~\ref{app:effects-of-gamma-image} for experimental results (Table~\ref{tab:effects-of-gamma-cifar10}) and detailed discussions.

\begin{wraptable}{r}{6cm}
\vspace{-10pt}
\caption{FID scores for CIFAR-10 conditional and CelebA unconditional image generation. We list comparisons with results reported by CD \cite{du2019implicit}, Noise-Conditional Score Network (NCSN) \cite{song2020generative} and $f$-EBMs \cite{yu2020training}. $\gamma = 1.0$ corresponds to maximizing spherical scoring rule (Example~\ref{example:spherical}).}
\begin{center}
\scalebox{0.9}{
\begin{tabular}{lc}
\toprule
Method & FID\\\midrule
\multicolumn{2}{c}{\textbf{CIFAR-10 ($\mathbf{32 \times 32}$) Conditional}} \\\midrule
Contrastive Divergence (KL) & 37.90\\
$f$-EBM (KL) & 37.36\\
$f$-EBM (Reverse KL) & 33.25\\
$f$-EBM (Squared Hellinger) & 32.19\\
$f$-EBM (Jensen Shannon) & 30.86\\\midrule
Pseudo-Spherical CD ($\gamma=2.0$) & 33.19\\
Pseudo-Spherical CD ($\gamma=1.0$) & \textbf{29.78}\\
Pseudo-Spherical CD ($\gamma=0.5$) & 35.02\\
Pseudo-Spherical CD ($\gamma=-0.5$) & \textbf{27.95}\\\midrule
\multicolumn{2}{c}{\textbf{CelebA ($\mathbf{64 \times 64}$)}} \\\midrule
Contrastive Divergence (KL) & 26.10\\
NCSN (w/o denoising) & 26.89\\
NCSN (w/ denoising) & 25.30\\
NCSNv2 (w/o denoising) & 28.86\\
NCSNv2 (w/ denoising) & \textbf{10.23}\\\midrule
Pseudo-Spherical CD ($\gamma=1.0$) & \textbf{24.76}\\
Pseudo-Spherical CD ($\gamma=-0.5$) & \textbf{20.35}\\
\bottomrule
\end{tabular}
}
\end{center}
\label{tab:image-generation}
\vspace{-20pt}
\end{wraptable}

\textbf{2-D Synthetic Data.}
For quantitative evaluation of the 2-D synthetic data experiments, we follow \cite{song2019bridging} and report the maximum mean discrepancy (MMD, \cite{borgwardt2006integrating}) between the generated samples and validation samples in Table~\ref{tab:synthetic} in App.~\ref{app:synthetic-data}, which demonstrates that PS-CD outperforms its CD counterpart on all but the Funnel dataset. From the histograms of samples shown in Figure~\ref{fig:hist2d} in App.~\ref{app:synthetic-data}, we can also have similar observations. For example, CD fails to place high densities in the center of the right mode in MoG, while PS-CD places the modes correctly.

\textbf{Image Generation.}
In Figure~\ref{fig:samples} in App.\ref{app:image-generation-samples}, we show MNIST, CIFAR-10 and CelebA samples produced by PS-CD (with $\gamma=1.0$), which demonstrate that our approach can produce highly realistic images with simple model architectures. As suggested in \cite{du2019implicit}, we use Fr\'echet Inception Distance (FID) \cite{heusel2017gans} as the quantitative evaluation metric for CIFAR-10 and CelebA, as Langevin dynamics converge to local minima that artificially inflate Inception Score \cite{salimans2016improved}. From Table~\ref{tab:image-generation}\footnote{For CelebA dataset, we reproduced the short-run MCMC method \cite{nijkamp2019learning} using our code base. Moreover, the $f$-EBM paper only reported results on CelebA 32x32 and we empirically found it is not comparable to our method in CelebA 64x64 (higher resolution), indicating better scalability of PS-CD to high-dimensional case.}, we can see that various members (different $\gamma$ values) in the family of PS-CD can outperform CD significantly, and more surprisingly, PS-CD also shows competitive performance to the recently proposed $f$-EBMs, without requiring expensive minimax optimization. While our method currently does not outperform the state-of-the-art image generation methods such as improved denoising score matching \cite{song2020generative}, which relies on carefully selected noise schedule and specially designed noise-dependent score network (modified U-Net architecture, hence not directly comparable to our results), we think that our work opens up a new research direction by bridging statistical decision theory (homogeneous proper scoring rules) and deep energy-based generative modeling. Moreover, under the setting of simple model architectures and the same hyperparameter configuration (\emph{e.g.}, batch size, learning rate, network structure, etc.), our empirical results suggest clear superiority of PS-CD over traditional CD and recent $f$-EBMs.

\begin{table}[t]
    \centering
    \caption{Robustness to data contamination on Gaussian datasets. The data distribution is $\mathcal{N}(-1, 0.5)$ and the contamination distribution is $\mathcal{N}(2, 0.05)$. We measure the KL divergence between clean target distribution $p$ and converged model distribution $q_\theta$, $\KL(p \| q_\theta)$.}
    \label{tab:contamination-gaussian}
    \vspace{5pt}
    \begin{tabular}{c c c c c}
    \toprule
      Contamination Ratio & CD & PS-CD ($\gamma=0.5$) & PS-CD ($\gamma=1.0$) & PS-CD ($\gamma=2.0$)\\
      \midrule
       0.01 & 0.0067 & 1e-5 & 1e-7 & 1e-6 \\
       0.05 & 0.0851 & 0.00027 & 1.6e-6 & 0.00011 \\
       0.1 & 0.1979 & 0.00173 & 1.86e-6 & 0.00012\\
       0.2 & 0.3869 & 0.1858 & 6.4e-6 & 0.00017 \\
       0.3 & 0.5438 & 0.5429 & 0.3118 & 0.00029\\
      \bottomrule
    \end{tabular}
    \vspace{-10pt}
\end{table}

\textbf{OOD Detection \& Robustness to Data Contamination.} We further test our methods on out-of-distribution (OOD) detection tasks. For the conditional CIFAR-10 model, we follow the evaluation protocol in \cite{du2019implicit}
and use $s(\vx) = \max_{y \in \mathcal{Y}} -E(\vx, y)$ as the score for detecting outliers. We use SVHN \cite{netzer2011reading}, Textures \cite{cimpoi2014describing}, Uniform/Gaussian Noise, CIFAR-10 Linear Interpolation and CelebA as the OOD datasets. We summarize the results in Table~\ref{tab:ood} in App.~\ref{app:ood}, from which we can see that PS-CD consistently outperforms CD and other likelihood-based models. 

Inspired by the OOD detection performance and previous work on robust parameter estimation under data contamination \cite{kanamori2015robust}, we further test the robustness of CD and PS-CD on both synthetic and natural image datasets. Specifically, suppose $p(\vx)$ is the underlying data distribution and there is another contamination distribution $\omega(\vx)$, e.g. uniform noise. In generative modeling under data contamination, our model 
observes i.i.d. samples from the contaminated distribution $\Tilde{p}(\vx) = c p(\vx) + (1-c) \omega(\vx)$, where $1-c \in [0, 1/2)$ is the contamination ratio. A theoretical advantage of pseudo-spherical score is its robustness to data contamination: the optimal solution of $\min_\vtheta S_{\text{ps}}(\Tilde{p}, q_\vtheta)$
is close to that of $\min_\vtheta S_{\text{ps}}(p, q_\vtheta)$
under some conditions (e.g. the density of $\omega(\vx)$ mostly lies in the region for which the target density $p(\vx)$ is small) \cite{fujisawa2008robust,kanamori2015robust}. From Table~\ref{tab:contamination-gaussian}, we can see that CD suffers from data contamination severely: as the contamination ratio increases, the performance degrades drastically. By contrast, PS-CD shows good robustness against data contamination and a larger $\gamma$ leads to better robustness. For example, PS-CD with $\gamma=1.0$ can properly approximate the target distribution when the contamination ratio is $0.2$, while PS-CD with $\gamma=2.0$ can do so when the contamination ratio is up to $0.3$.

We conduct similar experiments on MNIST and CIFAR-10 datasets, where we use uniform noise as the contamination distribution and the contamination ratio is 0.1 (i.e. 10\% images in the training set are replaced with random noise). After a warm-up pretraining\footnote{Note that it is impossible for a randomly initialized model to be robust to data contamination since without additional inductive bias, it will simply treat the contaminated distribution $\Tilde{p}(\vx)$ as the target.} (when the model has some OOD detection ability), we train the model with the contaminated data and measure the training progress.
We observe that CD gradually generates more random noise and diverge after a few training steps, while PS-CD is very robust. As shown in Table~\ref{tab:contamination-cifar} in App.~\ref{app:ood}, for a slightly pre-trained unconditional CIFAR-10 model (a simple 5-layer CNN with FID of 68.77), we observe that the performance of CD degrades drastically in terms of FID, while PS-CD can continuously improve the model even using the contaminated data. We provide visualizations and theoretical explanations in App.~\ref{app:ood}.
\vspace{-10pt}
\section{Conclusion}
\vspace{-10pt}
From the perspective of maximizing strictly proper homogeneous scoring rules, we propose pseudo-spherical contrastive divergence (PS-CD) to generalize maximum likelihood estimation of energy-based models. 
Different from prior works that involve joint training of auxiliary models \cite{yu2020training,gao2020flow,kumar2019maximum,dai2019exponential,grathwohl2020no}, PS-CD allows us to specify flexible modeling preferences without additional computational cost compared to contrastive divergence. We provide a theoretical analysis on the sample complexity and convergence property of the proposed method, as well as its connection to maximum likelihood. Finally, we demonstrate the effectiveness of PS-CD with extensive experiments on both synthetic data and commonly used image datasets.

\section*{Acknowledgements}
This research was supported by NSF(\#1651565, \#1522054, \#1733686), ONR (N000141912145), AFOSR (FA95501910024), ARO (W911NF-21-1-0125) and Sloan Fellowship.

\bibliographystyle{abbrv}
\bibliography{bibliography}

\newpage
\appendix
\section{Implementation of PS-CD Algorithm}\label{sec:app-alg}
In the following, we provide a simple PyTorch \cite{paszke2019pytorch} implementation of Step $6$ in Algorithm~\ref{alg:ps-cd}, where we directly differentiate through \texttt{log(mean(exp($\cdot$)))} for the first term and perform stop gradient operation on the self-normalized importance weight ($\exp(- \gamma E_\vtheta(\vx_i^-))$) for the second term. Mathematically, these two implementations are equivalent and we use different implementations here to illustrate the difference between their derivations. We further apply $L_2$ regularization on the outputs of the energy function to stabilize training.
\begin{Verbatim}[frame=single]
def logmeanexp(inputs):
    // Stable version log(mean(exp(inputs)))
    return inputs.max() + (inputs - inputs.max()).exp().mean().log()

def softmax(inputs):
    // Stable version softmax(inputs)
    exp_inputs = torch.exp(inputs - inputs.max())
    return exp_inputs / exp_inputs.sum()

def update_step(x_pos, x_neg, model_e, optim_e, l2_reg, gamma):
    // x_pos and x_neg are samples from p_data and q_theta.
    // model_e is the neural network for the energy function.
    // optim_e is the optimizer for model_e
    // e.g. optim_e = torch.optim.Adam(model_e.parameters())
    e_pos, e_neg = model_e(x_pos), model_e(x_neg)
    importance_weight = softmax(- gamma * e_neg)
    loss = - 1 / gamma * logmeanexp(- gamma * e_pos) - \
           torch.sum(e_neg * importance_weight.detach())
    loss += l2_reg * ((e_pos ** 2).mean() + (e_neg ** 2).mean())
    optim_e.zero_grad()
    loss_e.backward()
    optim_e.step()
\end{Verbatim}
\vfill

\newpage
\section{Proofs for Section~\ref{sec:method}}\label{app:proof-method}

\addtocounter{theorem}{-4}
\addtocounter{assumption}{-2}
\addtocounter{lemma}{-2}
\addtocounter{corollary}{-1}

\subsection{Proof for Theorem~\ref{the:ps-cd-gradient}}\label{app:proof-ps-cd-gradient}
In this section,
we provide two different ways to prove Theorem~\ref{the:ps-cd-gradient}. The first one is more straightforward and directly differentiates through the term $\log (\|\tq_\vtheta\|_{\gamma+1})$. The second one leverages a variational representation of $\log (\|\tq_\vtheta\|_{\gamma+1})$, where the optimal variational distribution happens to take an analytical form of $r^*_\vtheta(\vx) \propto \tq_\vtheta(\vx)^{\gamma + 1}$, thus avoiding the minimax optimization in other variational frameworks for KL and $f$-divergences \cite{yu2020training,dai2018kernel,dai2019exponential} and revealing the elegance of PS-CD.

\begin{theorem}
For an energy-based distribution $q_\vtheta \propto \tq_\vtheta = \exp(-E_\vtheta)$, the gradient of the negative $\gamma$-score $L_\gamma(\vtheta; p) = - S_\gamma(p, q_\vtheta)$ with respect to $\vtheta$ can be written as:
\begin{align}
    \nabla_\vtheta \gL_\gamma(\vtheta; p) = -\frac{1}{\gamma} \nabla_\vtheta \log \left(\mathbb{E}_{p(\vx)} [\exp(- \gamma E_\vtheta(\vx))]\right)
    - \mathbb{E}_{r_{\vtheta}(\vx)}[\nabla_\vtheta E_\vtheta(\vx)] \label{eq:app-ps-cd-gradient}
\end{align}
where the auxiliary distribution $r_\vtheta$ is also an energy-based distribution defined as: $$r_\vtheta(\vx) \defeq \frac{\tq_\vtheta(\vx)^{\gamma + 1}}{\int_\gX \tq_\vtheta(\vx)^{\gamma + 1} \mathrm{d} \vx} =  \frac{\exp(-(\gamma + 1) E_\vtheta(\vx))}{\int_\gX \exp(-(\gamma + 1) E_\vtheta(\vx)) \mathrm{d} \vx}.$$
\end{theorem}

\begin{proof}
\textbf{First proof: direct differentiation.} From Equation~(\ref{eq:loss-gamma}), we have:
\begin{align*}
    \nabla_\vtheta \gL_\gamma(\vtheta; p) &= \nabla_\vtheta \left(-\frac{1}{\gamma}  \log \left(\mathbb{E}_{p(\vx)} [\tq_\vtheta(\vx)^\gamma]\right) \nonumber
    +  \log(\|\tq_\vtheta\|_{\gamma + 1})\right) \\
    &= -\frac{1}{\gamma} \nabla_\vtheta \log \left(\mathbb{E}_{p(\vx)} [\exp(- \gamma E_\vtheta(\vx))]\right) + \frac{1}{\gamma + 1} \nabla_\vtheta \log \left(\int_\gX \exp(- (\gamma + 1) E_\vtheta(\vx)) \mathrm{d} \vx\right) \\
    &= -\frac{1}{\gamma} \nabla_\vtheta \log \left(\mathbb{E}_{p(\vx)} [\exp(- \gamma E_\vtheta(\vx))]\right) + \frac{1}{\gamma + 1} \frac{\int_\gX \exp(- (\gamma + 1) E_\vtheta(\vx)) \cdot (-(\gamma + 1) \nabla_\vtheta E_\vtheta(\vx))  \mathrm{d} \vx}{\int_\gX \exp(- (\gamma + 1) E_\vtheta(\vx)) \mathrm{d} \vx} \\
    &= -\frac{1}{\gamma} \nabla_\vtheta \log \left(\mathbb{E}_{p(\vx)} [\exp(- \gamma E_\vtheta(\vx))]\right) - \int_\gX \frac{\exp(- (\gamma + 1) E_\vtheta(\vx))}{\int_\gX \exp(- (\gamma + 1) E_\vtheta(\vy)) \mathrm{d} \vy} \nabla_\vtheta E_\vtheta(\vx) \mathrm{d} \vx \\
    &= -\frac{1}{\gamma} \nabla_\vtheta \log \left(\mathbb{E}_{p(\vx)} [\exp(- \gamma E_\vtheta(\vx))]\right)
    - \mathbb{E}_{r_{\vtheta}(\vx)}[\nabla_\vtheta E_\vtheta(\vx)]
\end{align*}

\textbf{Second proof: a variational representation with optimal variational distribution taking analytical form.} The main challenge is that the $\log \|\tq_\vtheta\|_{\gamma + 1}$ term in $\gL_\gamma(\vtheta; p)$ is generally intractable to compute. To solve this issue, we introduce the following variational representation:
\begin{lemma}\label{lemma:fenchel-dual}
Let $\Delta_\gX$ denote the set of all normalized probability density functions on sample space $\gX$.
With Fenchel duality, we have:
\begin{align}
    (\gamma + 1) \log(\|q\|_{\gamma+1}) 
    = \log \left( \int_\gX \tq(\vx)^{\gamma + 1}\mathrm{d}\vx \right)
    = \max_{r \in \Delta_\gX} \int_\gX r(\vx) \log\left(\tq(\vx)^{\gamma + 1}\right) \mathrm{d} \vx - \int_\gX r(\vx) \log r(\vx) \mathrm{d} \vx \label{eq:app-max_r_psi}
\end{align}
where the maximum is attained at $r^*(\vx) = \frac{\tq(\vx)^{\gamma + 1}}{\int_\gX \tq(\vx)^{\gamma + 1} \mathrm{d} \vx}$.
\end{lemma}
\begin{proof}
With Jensen's inequality, we have:
\begin{align*}
    \log \left( \int_\gX \tq(\vx)^{\gamma + 1}\mathrm{d}\vx \right) = \log \left(\int_\gX r(\vx) \frac{\tq(\vx)^{\gamma + 1}}{r(\vx)}\mathrm{d}\vx \right)
    \geq \int_\gX r(\vx) \log \left( \frac{\tq(\vx)^{\gamma + 1}}{r(\vx)} \right) \mathrm{d} \vx
\end{align*}

The equality holds if and only if:
\begin{align*}
    r^*(\vx) \propto \tq(\vx)^{\gamma + 1}
\end{align*}
As $r^* \in \Delta_\gX$ is a normalized distribution, we have:
\begin{align*}
    r^*(\vx) = \frac{\tq(\vx)^{\gamma + 1}}{\int_\gX \tq(\vx)^{\gamma + 1} \mathrm{d} \vx}
\end{align*}
\end{proof}

Now, suppose we parametrize the energy-based model distribution as $\tq_\vtheta = \exp(-E_\vtheta)$ and the variational distribution as $r_\vpsi$.
By plugging the variational representation in 
Lemma~\ref{lemma:fenchel-dual}
into $\gamma$-score (Equation~(\ref{eq:gamma-score})), we obtain the following minimax formulation to minimize the negative $\gamma$-score:
\begin{align*}
    \vpsi^*(\vtheta) = \argmax_\vpsi \gL_\gamma(\vtheta, \vpsi;p)
    \quad \vtheta^* = \argmin_\vtheta \max_\vpsi \gL_\gamma(\vtheta, \vpsi;p)
    = \argmin_\vtheta \gL_\gamma(\vtheta, \vpsi^*(\vtheta);p)
\end{align*}
where the game value function $L_\gamma(\vtheta, \vpsi;p)$ is defined as (\emph{s.t.} $L_\gamma(\vtheta, \vpsi^*(\vtheta);p) = - S_\gamma(p, q_\vtheta)$):
\begin{align*}
    L_\gamma(\vtheta, \vpsi;p ) = \underbrace{-\frac{1}{\gamma}  \log \left(\mathbb{E}_{p(\vx)} [\tq_\vtheta(\vx)^\gamma]\right)}_{L_1(\vtheta)}
    + \underbrace{\frac{1}{\gamma + 1} \left(\bb{E}_{r_\vpsi(\vx)} [\log\left(\tq_\vtheta(\vx)^{\gamma + 1}\right)] - \bb{E}_{r_\vpsi(\vx)} [\log r_\vpsi(x)]\right)}_{L_2(\vtheta, \vpsi)}
\end{align*}
By Lemma~\ref{lemma:fenchel-dual}, we know that $r_{\vpsi^*(\vtheta)} \propto \tq_\vtheta^{\gamma + 1}$.

The first term in Equation~(\ref{eq:app-ps-cd-gradient}) is simply $\nabla_\vtheta L_1(\vtheta)$. For the second term, since $L_2(\vtheta, \vpsi)$ is a function of both $\vtheta$ and $\vpsi$, and the optimal variational parameter $\vpsi^*(\vtheta)$ depends on $\vtheta$, the total derivative of $L_2(\vtheta, \vpsi^*(\vtheta))$ with respect to $\vtheta$ is:
\begin{align*}
    \frac{\mathrm{d} L_2(\vtheta, \vpsi^*(\vtheta))}{\mathrm{d} \vtheta}
    = \frac{\partial L_2(\vtheta, \vpsi^*(\vtheta))}{\partial \vtheta} + \frac{\partial L_2(\vtheta, \vpsi^*(\vtheta))}{\partial \vpsi^*(\vtheta)} \frac{\mathrm{d} \vpsi^*(\vtheta)}{\mathrm{d} \vtheta}
\end{align*}
Because $\vpsi^*(\vtheta)$ is the optimum of $L_2(\vtheta, \vpsi)$ (Lemma~\ref{lemma:fenchel-dual}), the second term in above equation is zero:
\begin{align*}
    \frac{\partial L_2(\vtheta, \vpsi^*(\vtheta))}{\partial \vpsi^*(\vtheta)} &= \frac{1}{\gamma + 1} \left(\int_\gX \nabla_\vpsi r_\vpsi(\vx) \log\left(\tq_\vtheta(\vx)^{\gamma + 1}\right) - \nabla_\vpsi r_\vpsi(\vx) \log r_\vpsi(\vx) - \frac{r_\vpsi(\vx)}{r_\vpsi(\vx)} \nabla_\vpsi r_\vpsi(\vx)  \mathrm{d} \vx
    \right)\at[\Bigg]{\vpsi = \vpsi^*(\vtheta)}\\
    &= \frac{1}{\gamma + 1}\left(\int_\gX (\log Z_{\vpsi^*(\vtheta)} - 1) \nabla_\vpsi r_\vpsi(\vx) \mathrm{d} \vx \right)\at[\Bigg]{\vpsi = \vpsi^*(\vtheta)}\\
    &= \frac{1}{\gamma + 1}\left((\log Z_{\vpsi^*(\vtheta)} - 1) \nabla_\vpsi \int_\gX  r_\vpsi(\vx) \mathrm{d} \vx \right) = 0
\end{align*}
where $Z_{\vpsi^*(\vtheta)}$ is the partition function of $r_{\vpsi^*(\vtheta)}$.

Thus we have:
\begin{align*}
    \frac{\mathrm{d} L_2(\vtheta, \vpsi^*(\vtheta))}{\mathrm{d} \vtheta}
    = \frac{\partial L_2(\vtheta, \vpsi^*(\vtheta))}{\partial \vtheta} =  - \mathbb{E}_{r_{\vpsi^*(\vtheta)}(\vx)}[\nabla_\vtheta E_\vtheta(\vx)]
\end{align*}
\end{proof}

\subsection{Proof for Theorem~\ref{the:importance-sampling}}\label{app:proof-importance-sampling}
\begin{theorem}[Consistent Gradient Estimation]
Let $\vx_1^+, \ldots, \vx_N^+$ be i.i.d. samples from $p(\vx)$ and $\vx_1^-, \ldots, \vx_N^-$ be i.i.d. samples from $q_\vtheta(\vx) \propto \exp(-E_\vtheta(\vx))$. Define the gradient estimator as:
\begin{align}
    \widehat{\nabla_\vtheta \gL_\gamma^N(\vtheta;p)} = 
    - \nabla_\vtheta \frac{1}{\gamma} \log \left( \frac{1}{N} \sum_{i=1}^N \exp(- \gamma E_\vtheta(\vx_i^+))\right) - 
    \frac{\sum_{i=1}^N \omega_\vtheta(\vx_i^-) \nabla_\vtheta E_\vtheta(\vx_i^-)}{\sum_{i=1}^N \omega_\vtheta(\vx_i^-)}
\end{align}
where the self-normalized importance weight $\omega_\vtheta(\vx_i^-) \defeq \overline{r}_{\vtheta}(\vx_i^-) / \overline{q}_\vtheta (\vx_i^-) = \exp(- \gamma E_\vtheta(\vx_i^-))$.
Then the gradient estimator converges to the true gradient in probability:
\begin{align*}
    \forall \epsilon > 0, \lim_{N \to \infty} \bb{P}\left(\left\|\widehat{\nabla_\vtheta \gL_\gamma^N(\vtheta;p)} - \nabla_\vtheta \gL_\gamma(\vtheta;p)\right\| \geq \epsilon \right) =0
\end{align*}
\end{theorem}
\begin{proof}
First, let us write $\widehat{\nabla_\vtheta \gL_\gamma^N(\vtheta;p)}$ and $\nabla_\vtheta \gL_\gamma(\vtheta;p)$ as:
\begin{align}
     &\nabla_\vtheta \gL_\gamma(\vtheta;p) = \frac{\mathbb{E}_{p(\vx)} [\exp(- \gamma E_\vtheta(\vx)) \nabla_\vtheta E_\vtheta(\vx)]}{\mathbb{E}_{p(\vx)} [\exp(- \gamma E_\vtheta(\vx))]} - \bb{E}_{r_{\vtheta}(\vx)}[\nabla_\vtheta E_\vtheta(\vx)] \label{eq:app-true-gradient}\\
    &\widehat{\nabla_\vtheta \gL_\gamma^N(\vtheta;p)} = \frac{\frac{1}{N} \sum_{i=1}^N \exp(- \gamma E_\vtheta(\vx_i^+)) \nabla_\vtheta E_\vtheta(\vx_i^+)}{\frac{1}{N} \sum_{i=1}^N \exp(- \gamma E_\vtheta(\vx_i^+))} - 
    \frac{\frac{1}{N}\sum_{i=1}^N \exp(- \gamma E_\vtheta(\vx_i^-)) \nabla_\vtheta E_\vtheta(\vx_i^-)}{\frac{1}{N} \sum_{i=1}^N \exp(- \gamma E_\vtheta(\vx_i^-))} \label{eq:app-estimate-gradient}
\end{align}
For the first term in Equation~(\ref{eq:app-estimate-gradient}), since $\{\vx_i^+\}_{i=1}^N$ are \emph{i.i.d.} samples from $p(\vx)$, by weak law of large numbers, the numerator and denominator of the first term in Equation~(\ref{eq:app-estimate-gradient}) converges to the numerator and denominator of the first term in Equation~(\ref{eq:app-true-gradient}) in probability. By Slutsky's theorem (\emph{i.e.}, for random variables $X_N,X,Y_N,Y$, if $X_N \overset{p}{\to} X, Y_N \overset{p}{\to} Y$ and $X$, $Y$ are constants, then $X_N/Y_N \overset{p}{\to} X/Y$), the first term of Equation~(\ref{eq:app-estimate-gradient}) converges to the first term of Equation~(\ref{eq:app-true-gradient}) in probability.

Let us use $Z_r$ and $Z_q$ to denote the partition function for $r_{\vtheta}$ and $q_\vtheta$. The second term of Equation~(\ref{eq:app-estimate-gradient}) can be written as:
\begin{align}
    \frac{\frac{1}{N}\sum_{i=1}^N \exp(- \gamma E_\vtheta(\vx_i^-)) \nabla_\vtheta E_\vtheta(\vx_i^-)}{\frac{1}{N} \sum_{i=1}^N \exp(- \gamma E_\vtheta(\vx_i^-))} = 
    \frac{\frac{1}{N}\sum_{i=1}^N \frac{\exp(-(\gamma+1)E_\vtheta(\vx_i^-) / Z_r}{\exp(-E_\vtheta(\vx_i^-))/Z_q} \nabla_\vtheta E_\vtheta(\vx_i^-)}{\frac{1}{N}\sum_{i=1}^N \frac{\exp(-(\gamma+1)E_\vtheta(\vx_i^-) / Z_r}{\exp(-E_\vtheta(\vx_i^-))/Z_q}}\label{eq:app-expand-importance}
\end{align}
Since $\{\vx_i^-\}_{i=1}^N$ are \emph{i.i.d.} samples from $q_\vtheta(\vx)$, the numerator of Equation~(\ref{eq:app-expand-importance}) converges to $\bb{E}_{r_\vtheta(\vx)}[\nabla_\vtheta E_\vtheta(\vx)]$ in probability, while the denominator of Equation~(\ref{eq:app-expand-importance}) converges to $1$ in probability ($\bb{E}_{q_\vtheta(\vx)}[r_\vtheta(\vx)/q_\vtheta(\vx)] = 1$). By Slutsky's theorem, the second term of Equation~(\ref{eq:app-estimate-gradient}) converges to the second term of Equation~(\ref{eq:app-true-gradient}) in probability. Furthermore, since convergence in probability is also preserved under addition transformation, the gradient estimator in Equation~(\ref{eq:app-estimate-gradient}) converges to the true gradient in Equation~(\ref{eq:app-true-gradient}) in probability.
\end{proof}

\newpage
\subsection{Connections to Maximum Likelihood Estimation and Extension to $\gamma<0$}
\begin{lemma}\label{lemma:app-connection-mle}
Let $D_\gamma(p,q)$ be the divergence corresponding to $\gamma$-scoring rule. Then, we have:
\begin{align*}
    \lim_{\gamma \to 0} D_\gamma(p,q) = \KL(p\|q)
\end{align*}
\end{lemma}
\begin{proof}
As introduced in Equation~(\ref{eq:entropy-divergence}) in Section~\ref{sec:proper-scoring-rules}, the divergence corresponding to the $\gamma$-scoring rule is:
\begin{align*}
    D_\gamma(p, q) &= S_\gamma(p, p) - S_\gamma(p, q) \\
    &= \frac{1}{\gamma}  \log \left(\mathbb{E}_{p(\vx)} [p(\vx)^\gamma]\right) - \log(\|p\|_{\gamma + 1}) - \frac{1}{\gamma}  \log \left(\mathbb{E}_{p(\vx)} [q(\vx)^\gamma]\right) + \log(\|q\|_{\gamma + 1}) \\
    &= - \frac{1}{\gamma}  \log \left(\int_\gX p(\vx)q(\vx)^\gamma \mathrm{d} \vx\right) + \frac{1}{\gamma + 1} \log \left(\int_\gX q(\vx)^{\gamma + 1} \mathrm{d} \vx\right) + \frac{1}{\gamma (\gamma + 1)} \log \left(\int_\gX p(\vx)^{\gamma + 1} \mathrm{d} \vx \right)
\end{align*}

When $\gamma \to 0$, with Taylor series, we know that:
\begin{align*}
    q^\gamma = 1 + \gamma \log (q) + \gO(\gamma^2)\\
    p^\gamma = 1 + \gamma \log (p) + \gO(\gamma^2)
\end{align*}

Therefore, we have:
\begin{align*}
    \lim_{\gamma \to 0} D_\gamma(p,q) &= \lim_{\gamma \to 0} - \frac{1}{\gamma}  \log \left(\int_\gX p(\vx)(1 + \gamma \log q(\vx) + \gO(\gamma^2)) \mathrm{d}\vx\right) \\
    &~~~~~~~~~~~~+ \frac{1}{\gamma + 1} \log \left(\int_\gX q(\vx)(1+\gamma \log q(\vx)+\gO(\gamma^2)) \mathrm{d} \vx\right) \\
    &~~~~~~~~~~~~+\frac{1}{\gamma (\gamma + 1)} \log \left(\int_\gX p(\vx)(1+\gamma \log p(\vx) + \gO(\gamma^2))\mathrm{d}\vx\right)\\
    &=\lim_{\gamma \to 0} -\frac{1}{\gamma} \log\left(1 + \gamma \int_\gX p(\vx)\log q(\vx) \mathrm{d}\vx + \gO(\gamma^2)\right) \\
    &~~~~~~~~~~~~+ \frac{1}{\gamma + 1}\log \left(1 + \gamma \int_\gX q(\vx)\log q(\vx) \mathrm{d}\vx + \gO(\gamma^2) \right)\\
    &~~~~~~~~~~~~+ \frac{1}{\gamma (\gamma + 1)} \log \left(1 + \gamma\int_\gX p(\vx)\log p(\vx)\mathrm{d}\vx + \gO(\gamma^2))\right)\\
    &= \lim_{\gamma \to 0} -\int_\gX p(\vx)\log q(\vx) \mathrm{d}\vx + \frac{1}{\gamma + 1}\int_\gX p(\vx) \log p(\vx) \mathrm{d}\vx + \gO(\gamma)\\
    &= \int_\gX p(\vx) \log \frac{p(\vx)}{q(\vx)} = \KL(p\|q)
\end{align*}
\end{proof}
The above lemma implies that the KL divergence minimization (maximum likelihood estimation) is a special case of $\gamma$-divergence minimization when $\gamma \to 0$, which also implies the following corollary:
\begin{corollary}
When $\gamma \to 0$, the gradient of pseudo-spherical contrastive divergence is equal to the gradient of contrastive divergence:
\begin{align*}
    \lim_{\gamma \to 0} \nabla_\vtheta \gL_\gamma(\vtheta;p) = \nabla_\vtheta \gL_{\mathrm{MLE}}(\vtheta; p)
\end{align*}
\end{corollary}
\begin{proof}
This is a direct consequence of Lemma~\ref{lemma:app-connection-mle}.
It can also be verified by checking the PS-CD gradient in Equation~(\ref{eq:ps-cd-gradient}) (when $\gamma \to 0$, $r_\vtheta = q_\vtheta \propto \exp(-E_\vtheta)$):
\begin{align*}
    \lim_{\gamma \to 0} \nabla_\vtheta \gL_\gamma(\vtheta;p) &= \lim_{\gamma \to 0} -\frac{1}{\gamma} \nabla_\vtheta \log \left(\mathbb{E}_{p(\vx)} [\exp(- \gamma E_\vtheta(\vx))]\right) - \mathbb{E}_{r_\vtheta(\vx)}[\nabla_\vtheta E_\vtheta(\vx)] \\
    &= \lim_{\gamma \to 0} - \frac{1}{\gamma} \frac{- \gamma \mathbb{E}_{p(\vx)} [\exp(- \gamma E_\vtheta(\vx)) \nabla_\vtheta E_\vtheta(\vx)]}{\mathbb{E}_{p(\vx)} [\exp(- \gamma E_\vtheta(\vx))]} - \mathbb{E}_{r_\vtheta(\vx)}[\nabla_\vtheta E_\vtheta(\vx)] \\
    &= \bb{E}_{p(\vx)}[\nabla_\vtheta E_\vtheta(\vx)] - \bb{E}_{q_\vtheta(\vx)}[\nabla_\vtheta E_\vtheta(\vx)] = \nabla_\vtheta \gL_{\mathrm{MLE}}(\vtheta; p)
\end{align*}
\end{proof}

Inspired by \cite{van2014renyi,li2016renyi} that generalize R\'enyi divergence beyond its definition to negative orders, we consider the extension of $\gamma$-scoring rule with $\gamma < 0$ (although it is no longer strictly proper for these $\gamma$ values) and show that maximizing such scoring rule is equivalent to maximizing a lower bound of logarithm scoring rule (MLE) with an additional R\'enyi entropy regularization.

\begin{lemma}
When $-1 \leq \gamma < 0$, we have:
\begin{align*}
    S_\gamma(p, q) \leq \bb{E}_{p(\vx)} [\log q(\vx)] + \frac{\gamma
    }{\gamma + 1} \gH_{\gamma+1}(q)
\end{align*}
where $\gH_{\gamma + 1} (q)$ is the R\'enyi entropy of order $\gamma + 1$.
\end{lemma}
\begin{proof}
As a generalization to Shannon entropy, the R\'enyi entropy of order $\alpha$ is defined as:
\begin{align*}
    \gH_\alpha(q) = \frac{\alpha}{1 - \alpha} \log(\|q\|_\alpha)
\end{align*}
With Jensen's inequality, for $-1 \leq \gamma < 0$, we have:
\begin{align*}
    S_\gamma(p, q) &= \frac{1}{\gamma} \log (\bb{E}_{p(\vx)} [q(\vx)^\gamma]) - \log (\|q\|_{\gamma+1}) \\
    & \leq \frac{1}{\gamma} \bb{E}_{p(\vx)}[\gamma \log (q(\vx))] + \frac{\gamma}{\gamma + 1} \left(\frac{\gamma + 1}{- \gamma} \log (\|q\|_{\gamma + 1})\right) \\
    & = \bb{E}_{p(\vx)} [\log q(\vx)] + \frac{\gamma
    }{\gamma + 1} \gH_{\gamma+1}(q)
\end{align*}
\end{proof}

\section{Theoretical Analysis}\label{app:theoretical-analysis}
In this section, we provide a theoretical analysis on the sample complexity of the gradient estimator, as well as the convergence property of stochastic gradient descent with consistent (but biased given finite samples) gradient estimators as presented in Algorithm~\ref{alg:ps-cd}.

\subsection{Sample Complexity}\label{sec:app-sample-complexity}
We start with analyzing the sample complexity of the consistent gradient estimator, that is how fast it approaches the true gradient value or how many samples we need in order to empirically estimate the gradient at a given accuracy with a high probability.

We first make the following assumption, which is similar to the one used in \cite{belghazi2018mine,kohler2017sub}:
\begin{assumption}\label{assump:app-energy}
The energy function is bounded by $K$ and the gradient is bounded by $L$ (with $K>0$ and $L>0$):
$$\forall \vx \in \gX, ~\vtheta \in \Theta, ~|E_\vtheta(\vx)| \leq K, ~\|\nabla_\vtheta E_\vtheta(\vx)\| \leq L.$$
\end{assumption}
The assumption is typically easy to enforce in practice. For example, in the experiments we use $L_2$ regularization on the outputs of the energy function, as well as normalized inputs and spectral normalization \cite{miyato2018spectral} for the neural network that realizes the energy function.

\begin{theorem}\label{the:app-sample-complexity}
Under Assumption~\ref{assump:app-energy}, given any constants $\epsilon > 0$ and $\delta \in (0,1)$, when the number of samples $N$ satisfies:
\begin{align*}
    N \geq \frac{32 L^2 e^{8 \gamma K}\left(1 + 4\log(2/\delta)\right)}{\epsilon^2}
\end{align*}
we have:
\begin{align*}
    \bb{P}\left(\left\|\widehat{\nabla_\vtheta \gL_\gamma^N(\vtheta;p)} - \nabla_\vtheta \gL_\gamma(\vtheta;p)\right\| \leq \epsilon \right) \geq 1 - \delta
\end{align*}
\end{theorem}
\begin{proof}
For notation simplicity, we use $p^N$ to denote the empirical distribution of $\{\vx_i\}_{i=1}^N$ \emph{i.i.d.} sampled from a distribution $p$, \emph{i.e.}, $\bb{E}_{p^N(\vx)} [f(\vx)] = \frac{1}{N} \sum_{i=1}^N f(\vx_i)$. Similarly, $\bb{E}_{q_\vtheta^N(\vx)} [f(\vx)] = \sum_{i=1}^N f(\vx_i)$ when $\{\vx_i\}_{i=1}^N$ are \emph{i.i.d.} samples from $q_\vtheta$.

First, we observe that:
\begin{align*}
    \bb{E}_{r_\vtheta(\vx)} [\nabla_\vtheta \Et] = \frac{\bb{E}_{q_\vtheta(\vx)} [\frac{r_\vtheta(\vx)}{q_\vtheta(\vx)} \nabla_\vtheta \Et]}{\bb{E}_{q_\vtheta(\vx)}[\frac{r_\vtheta(\vx)}{q_\vtheta(\vx)}]} = \frac{\bb{E}_{q_\vtheta(\vx)} [\exp(- \gamma \Et) \nabla_\vtheta \Et]}{\bb{E}_{q_\vtheta(\vx)} [\exp(- \gamma \Et)]}
\end{align*}
where the partition functions of $r_\vtheta$ and $q_\vtheta$ cancel out.
Based on Equation~(\ref{eq:app-true-gradient}) and (\ref{eq:app-estimate-gradient}), with triangle inequality, the estimation error can be upper bounded as:
\begin{align}
     &~ \left\|\widehat{\nabla_\vtheta \gL_\gamma^N(\vtheta;p)} - \nabla_\vtheta \gL_\gamma(\vtheta;p)\right\| \nonumber\\
     \leq &~ \underbrace{\left\| \frac{\mathbb{E}_{p(\vx)} [\exp(- \gamma E_\vtheta(\vx)) \nabla_\vtheta E_\vtheta(\vx)]}{\mathbb{E}_{p(\vx)} [\exp(- \gamma E_\vtheta(\vx))]} - \frac{\mathbb{E}_{p^N(\vx)} [\exp(- \gamma E_\vtheta(\vx)) \nabla_\vtheta E_\vtheta(\vx)]}{\mathbb{E}_{p^N(\vx)} [\exp(- \gamma E_\vtheta(\vx))]} \right\|}_{\Delta_p} + \label{eq:app-sample-complexity}\\
     &~ \underbrace{\left\| \frac{\mathbb{E}_{q_\vtheta(\vx)} [\exp(- \gamma E_\vtheta(\vx)) \nabla_\vtheta E_\vtheta(\vx)]}{\mathbb{E}_{q_\vtheta(\vx)} [\exp(- \gamma E_\vtheta(\vx))]} - \frac{\mathbb{E}_{q_\vtheta^N(\vx)} [\exp(- \gamma E_\vtheta(\vx)) \nabla_\vtheta E_\vtheta(\vx)]}{\mathbb{E}_{q_\vtheta^N(\vx)} [\exp(- \gamma E_\vtheta(\vx))]} \right\|}_{\Delta_{q_\vtheta}}\nonumber
\end{align}
Define functions: 
$$f_\vtheta(\vx) \defeq \exp(- \gamma E_\vtheta(\vx)),~~\vh_\vtheta(\vx) \defeq \exp(- \gamma E_\vtheta(\vx)) \nabla_\vtheta E_\vtheta(\vx)$$ 
From Assumption~\ref{assump:app-energy}, we know that:
\begin{align}
    \forall \vx \in \gX, \vtheta \in \Theta, f_\vtheta(\vx) \in [e^{-\gamma K}, e^{\gamma K}], \|\vh_\vtheta(\vx)\| \leq L e^{\gamma K} \label{eq:app-g-h-condition}
\end{align}

Let us examine the first term $\Delta_p$ in Equation~(\ref{eq:app-sample-complexity}):
\begin{align}
    \Delta_p =&~ \left\| \frac{\bb{E}_{p(\vx)} [\vh_\vtheta(\vx)]}{\bb{E}_{p(\vx)}[f_\vtheta(\vx)]} - \frac{\bb{E}_{p^N(\vx)} [\vh_\vtheta(\vx)]}{\bb{E}_{p^N(\vx)}[f_\vtheta(\vx)]} \right\| \nonumber\\
    =&~ \frac{1}{\bb{E}_{p(\vx)} [f_\vtheta(\vx)] \cdot \bb{E}_{p^N(\vx)} [f_\vtheta(\vx)]} \left\| \bb{E}_{p^N(\vx)} [f_\vtheta(\vx)] \cdot \bb{E}_{p(\vx)} [\vh_\vtheta(\vx)] - \bb{E}_{p(\vx)} [f_\vtheta(\vx)] \cdot \bb{E}_{p^N(\vx)} [\vh_\vtheta(\vx)] \right\| \nonumber\\
    \leq&~ e^{2\gamma K} \left\| \bb{E}_{p^N(\vx)} [f_\vtheta(\vx)] \cdot \bb{E}_{p(\vx)} [\vh_\vtheta(\vx)] - \bb{E}_{p(\vx)} [f_\vtheta(\vx)] \cdot \bb{E}_{p^N(\vx)} [\vh_\vtheta(\vx)] \right\| \label{eq:app-delta-p-upper-bound}
\end{align}

Now we introduce the following lemma that will provide us a probability upper bound that an empirical mean of independent random variables deviates from its expected value more than a certain amount.

\begin{lemma}[Vector Bernstein Inequality \cite{kohler2017sub,gross2011recovering}]\label{lemma:bernstein} Let $\bm{X}_1, \ldots, \bm{X}_N$ be independent vector-valued random variables. Assume that each one is centered, uniformly bounded and the variance is also bounded:
$$\forall i, \bb{E} [\bm{X}_i]=0 ~\text{and}~ \|\bm{X}_i\| \leq \mu ~\text{and}~ \bb{E}[\|\bm{X}_i\|^2] \leq \sigma^2$$
Define $\overline{\bm{X}} \defeq \frac{1}{N}(\bm{X}_1 + \ldots + \bm{X}_N)$. Then we have for $0 < t < \sigma^2/\mu$:
\begin{align*}
    \bb{P}(\|\overline{\bm{X}}\| \geq t) \leq \exp \left( - \frac{N t^2}{8 \sigma^2} + \frac{1}{4} \right)
\end{align*}
\end{lemma}

We then define the following vector-valued random variable:
\begin{align*}
    \bm{X}_i \defeq & f_\vtheta(\vx_i) \cdot \bb{E}_{p(\vx)}[\vh_\vtheta(\vx)] - \bb{E}_{p(\vx)}[f_\vtheta(\vx)] \cdot \vh_\vtheta(\vx_i)\\
    \overline{\bm{X}} \defeq & \frac{1}{N}\sum_{i=1}^N f_\vtheta(\vx_i) \cdot \bb{E}_{p(\vx)}[\vh_\vtheta(\vx)] - \bb{E}_{p(\vx)}[f_\vtheta(\vx)] \cdot \frac{1}{N}\sum_{i=1}^N \vh_\vtheta(\vx_i)\\
    = & \bb{E}_{p^N(\vx)} [f_\vtheta(\vx)] \cdot \bb{E}_{p(\vx)} [\vh_\vtheta(\vx)] - \bb{E}_{p(\vx)} [f_\vtheta(\vx)] \cdot \bb{E}_{p^N(\vx)} [\vh_\vtheta(\vx)]
\end{align*}

From Equation~(\ref{eq:app-g-h-condition}), we know that:
\begin{align*}
    &\|\bm{X}_i\| = \|f_\vtheta(\vx_i) \cdot \bb{E}_{p(\vx)}[\vh_\vtheta(\vx)] - \bb{E}_{p(\vx)}[f_\vtheta(\vx)] \cdot \vh_\vtheta(\vx_i)\| \\
    &~~~~~~~~~\leq \|f_\vtheta(\vx_i) \cdot \bb{E}_{p(\vx)}[\vh_\vtheta(\vx)] \| + \|\bb{E}_{p(\vx)}[f_\vtheta(\vx)] \cdot \vh_\vtheta(\vx_i) \| \leq 2L e^{2 \gamma K} \\
    &\|\bm{X}_i\|^2 \leq 4 L^2 e^{4 \gamma K}
\end{align*}

With $\sigma^2 \defeq 4 L^2 e^{4 \gamma K}$, from Lemma~\ref{lemma:bernstein}, we know that:
\begin{align}
    &~ \bb{P}\left(e^{2\gamma K} \left\| \bb{E}_{p^N(\vx)} [f_\vtheta(\vx)] \cdot \bb{E}_{p(\vx)} [\vh_\vtheta(\vx)] - \bb{E}_{p(\vx)} [f_\vtheta(\vx)] \cdot \bb{E}_{p^N(\vx)} [\vh_\vtheta(\vx)] \right\| \geq \frac{\epsilon}{2}\right) \nonumber\\
    \leq
    &~ \exp \left( - \frac{N \epsilon^2}{128 L^2 e^{8 \gamma K}} + \frac{1}{4} \right) \label{eq:app-prob-bound-p}
\end{align}

To obtain a sample complexity bound such that the probability bound in Equation~(\ref{eq:app-prob-bound-p}) is less than $1-\sqrt{1-\delta}$, we need to solve for $N$:
\begin{align}
    \exp \left( - \frac{N \epsilon^2}{128 L^2 e^{8 \gamma K}} + \frac{1}{4} \right) \leq 1-\sqrt{1-\delta} \label{eq:app-delta-p-sample}
\end{align}

Solving Equation~(\ref{eq:app-delta-p-sample}) gives us:
\begin{align}
N \geq \frac{32 L^2 e^{8 \gamma K}\left(1 - 4\log\left(1 - \sqrt{1 - \delta}\right)\right)}{\epsilon^2}
\end{align}

Because $1 + 4 \log(2/\delta) > 1 - 4 \log(1 - \sqrt{1 - \delta})$ for $\delta \in (0,1]$, we use the following slightly weaker bound such that it looks cleaner:
\begin{align}
N \geq \frac{32 L^2 e^{8 \gamma K}\left(1 + 4\log(2/\delta)\right)}{\epsilon^2} \label{eq:app-sample-complexity-delta-p}
\end{align}

Since Equation~(\ref{eq:app-delta-p-upper-bound}) is an upper bound of $\Delta_p$, we know that when the sample size satisfies Equation~(\ref{eq:app-sample-complexity-delta-p}), we have:
\begin{align*}
    &~ \bb{P}\left( \Delta_p = \left\| \frac{\bb{E}_{p(\vx)} [\vh_\vtheta(\vx)]}{\bb{E}_{p(\vx)}[f_\vtheta(\vx)]} - \frac{\bb{E}_{p^N(\vx)} [\vh_\vtheta(\vx)]}{\bb{E}_{p^N(\vx)}[f_\vtheta(\vx)]} \right\| \leq \frac{\epsilon}{2} \right)\\
    \geq &~ \bb{P}\left(e^{2\gamma K} \left\| \bb{E}_{p^N(\vx)} [f_\vtheta(\vx)] \cdot \bb{E}_{p(\vx)} [\vh_\vtheta(\vx)] - \bb{E}_{p(\vx)} [f_\vtheta(\vx)] \cdot \bb{E}_{p^N(\vx)} [\vh_\vtheta(\vx)] \right\| \leq \frac{\epsilon}{2}\right) \\
    \geq &~ \sqrt{1 - \delta}
\end{align*}

Similarly, we can obtain the same sample complexity bound for $\Delta_{q_\vtheta}$ such that:
\begin{align}
    \bb{P}\left( \Delta_{q_\vtheta} = \left\| \frac{\bb{E}_{q_\vtheta(\vx)} [\vh_\vtheta(\vx)]}{\bb{E}_{q_\vtheta(\vx)}[f_\vtheta(\vx)]} - \frac{\bb{E}_{q_\vtheta^N(\vx)} [\vh_\vtheta(\vx)]}{\bb{E}_{q_\vtheta^N(\vx)}[f_\vtheta(\vx)]} \right\| \leq \frac{\epsilon}{2} \right) \geq \sqrt{1 - \delta} \label{eq:app-sample-complexity-delta-q}
\end{align}

From Equation~(\ref{eq:app-sample-complexity}), we know that $\Delta_p + \Delta_{q_\vtheta}$ is an upper bound of the gradient estimation error. Also note that the event $\Delta_p \leq \frac{\epsilon}{2}$ and the event $\Delta_{q_\vtheta} \leq \frac{\epsilon}{2}$ are independent from each other (the samples for $p^N$ and the samples for $q_\vtheta^N$ are independent samples from $p$ and $q_\vtheta$ respectively). Thus when the sample size satisfies Equation~(\ref{eq:app-sample-complexity-delta-p}), we have:
\begin{align*}
&~\bb{P}\left(\|\widehat{\nabla_\vtheta \gL_\gamma^N(\vtheta;p)} - \nabla_\vtheta \gL_\gamma(\vtheta;p)\| \leq \epsilon \right) \\
\geq &~\bb{P}\left(\Delta_p + \Delta_{q_\vtheta} \leq \epsilon \right)\\ \geq &~\bb{P}\left(\Delta_p \leq \frac{\epsilon}{2} ~\text{and}~ \Delta_{q_\vtheta} \leq \frac{\epsilon}{2}\right) \\
=&~\bb{P}\left(\Delta_p \leq \frac{\epsilon}{2}\right) \cdot \bb{P}\left(\Delta_{q_\vtheta} \leq \frac{\epsilon}{2}\right) \\
\geq &~1 - \delta
\end{align*}
\end{proof}

\subsection{Convergence of Pseudo-Spherical Contrastive Divergence Algorithm}\label{app:ps-cd-convergence}
In this section, we analyze the convergence property of the PS-CD algorithm presented in Algorithm~\ref{alg:ps-cd}. 
For notation simplicity, we define $\vg$ as the true gradient in Equation~(\ref{eq:ps-cd-gradient}) and $\widehat{\vg}$ as the gradient estimator in Equation~(\ref{eq:gradient-estimator}). We further use $\gL(\vtheta)$ to denote the loss function $\gL_\gamma(\vtheta, \vpsi^*(\vtheta);p) = - S_\gamma(p, q_\vtheta)$.

Let us consider the following stochastic gradient descent (SGD) update rule:
\begin{align}
    \vtheta_{t+1} = \vtheta_t - \eta_t \widehat{\vg}_t,~~t=1,2,\ldots,T \label{eq:app-update-rule}
\end{align}
where $\eta_t$ is the step size at step $t$, $T$ is the total number of steps and $\widehat{\vg}_t \defeq \widehat{\nabla_\vtheta \gL_\gamma^N(\vtheta;p)}\at[\big]{\vtheta=\vtheta_t}$ is the consistent (but biased) gradient estimation of $\vg_t$ at step $t$. Note that $\vtheta_t$ and $\widehat{\vg}_t$ are random variables that depend on the previous history $\widehat{\vg}_1,\ldots,\widehat{\vg}_{t-1}$. For brevity, in the following we will omit such dependency in the notations.

Most works for analyzing the convergence behavior of SGD relies on the assumption that the gradient estimator $\widehat{\vg}_t$ is asymptotically unbiased, \emph{e.g.}, \cite{nemirovski2009robust,lacoste2012simpler,shalev2014understanding,ghadimi2013stochastic,reddi2016stochastic}, while in our case the gradient estimator is not unbiased but consistent (see Section 1.2 in \cite{chen2018stochastic} for a detailed discussion on the distinctions between unbiasedness and consistency). Therefore, in this work we generalize the theory developed in \cite{ghadimi2013stochastic} and \cite{chen2018stochastic} to analyze the convergence rate for PS-CD.

Besides Assumption~\ref{assump:app-energy} used for analyzing the sample complexity of the gradient estimator, we further make the following assumption:
\begin{assumption}\label{assump:app-smooth}
The loss function $\gL(\vtheta)$ is $M$-smooth (with $M>0$):
$$\forall \vtheta_1, \vtheta_2 \in \Theta, ~\|\nabla \gL(\vtheta_1) - \nabla \gL(\vtheta_2)\| \leq M \|\vtheta_1 - \vtheta_2\|.$$
\end{assumption}
This is a common assumption used for analyzing first-order optimization methods, which is also used in \cite{ghadimi2013stochastic,chen2018stochastic}. Also note that this is a relatively mild assumption since we do not require the loss function to be convex in $\vtheta$. Since in non-convex optimization, the convergence criterion is typically measured by gradient norm, following \cite{nesterov2013introductory,ghadimi2013stochastic}, we use $\|\nabla \gL(\vtheta)\| \leq \xi$ to judge whether a solution $\vtheta$ is approximately a stationary point.

Now, let us consider Algorithm~\ref{alg:randomized-sgd}, which is a variant of SGD that allows early stopping before reaching the iteration limit $T$ according to some probability distribution $p_Z$ over iteration indexes $[T] \defeq \{1,\ldots,T\}$.

\begin{algorithm}[h]
   \caption{Randomized Stochastic Gradient Descent}
   \label{alg:randomized-sgd}
\begin{algorithmic}[1]
   \STATE {\bfseries Input:} Initial parameter $\vtheta_1$, iteration limit $T$, step sizes $\{\eta_t\}_{t=1}^T$, distribution $p_Z$ over $[T]$.
   \STATE Sample an iteration number $Z$ from $p_Z$ (defined in Equation~(\ref{eq:app-def-pz})).
   \FOR{$t=1,\ldots,Z$}
   \STATE Obtain the gradient estimator $\widehat{\vg}_t$ with a sample batch size of $N_t$.
   \STATE Update the parameter: $\vtheta_{t+1} = \vtheta_t - \eta_t \widehat{\vg}_t$.
   \ENDFOR
   \STATE {\bfseries Output: $\vtheta_Z$}.
\end{algorithmic}
\end{algorithm}
Note that this is equivalent (more efficient in terms of computation) to running the algorithm to the iteration limit $T$ and then selecting the final solution from $\{\vtheta_1,\ldots,\vtheta_T\}$ according to distribution $p_Z$.

We have the following theorem that characterizes the convergence property of Algorithm~\ref{alg:randomized-sgd}:
\begin{theorem}\label{the:app-sgd-convergence}
Under Assumptions~\ref{assump:app-energy} and \ref{assump:app-smooth}, for arbitrary constants $\alpha \in (0,1)$ and $\delta \in (0,1)$, suppose that the step sizes satisfy $\eta_t < 2(1-\alpha)/M$ and the probability distribution over iteration indexes is chosen to be:
\begin{align}
    p_Z(t) \defeq \frac{2(1 - \alpha)\eta_t - M \eta_t^2}{\sum_{t=1}^T (2(1-\alpha)\eta_t - M \eta_t^2)},~~t=1,\ldots,T
    \label{eq:app-def-pz}
\end{align}
and the sample size $N_t$ used for estimating $\widehat{\vg}_t$ satisfies:
\begin{align}
    N_t \geq \frac{32 L^2 e^{8 \gamma K} (1 + 4 \log(2T/\delta))}{\alpha^2 \|\vg_t\|^2} \label{eq:app-convergence-sample}
\end{align}
Denote by $\gL^*$ the minimum value of $\gL(\vtheta)$. Then with probability at least $1-\delta$, we have:
\begin{align}
    \bb{E}_{p_Z}[\|\nabla \gL(\vtheta_Z)\|^2] < \frac{2(\gL(\vtheta_1) - \gL^*) + 12 \alpha M L^2 e^{4 \gamma K} \sum_{t=1}^T \eta_t^2}{\sum_{t=1}^T (2(1-\alpha)\eta_t - M \eta_t^2)} \label{eq:sgd-converge-gradient-bound}
\end{align}
\end{theorem}

\begin{proof}
First, with Assumption~\ref{assump:energy} ($|E_\vtheta(\vx)| \leq K, ~\|\nabla_\vtheta E_\vtheta(\vx)\| \leq L$), we can bound the norm of the true gradient as:
\begin{align}
    \|\vg\| &\leq \left\|\frac{\mathbb{E}_{p(\vx)} [\exp(- \gamma E_\vtheta(\vx)) \nabla_\vtheta E_\vtheta(\vx)]}{\mathbb{E}_{p(\vx)} [\exp(- \gamma E_\vtheta(\vx))]}\right\| + \|\bb{E}_{r_{\vpsi^*(\vtheta)}(\vx)}[\nabla_\vtheta E_\vtheta(\vx)]\| \nonumber\\
    &\leq L e^{2\gamma K} + L < 2L e^{2\gamma K} \label{eq:app-gradient-norm-bound}
\end{align}

From Theorem~\ref{the:sample-complexity}, we know that when sample size at each step satisfies Equation~(\ref{eq:app-convergence-sample}), we have:
\begin{align*}
    \bb{P}(\|\widehat{\vg}_t - \vg_t\| \leq \alpha \|\vg_t\|) \geq 1 - \delta / T
\end{align*}
Therefore, we have:
\begin{align*}
    \bb{P}(\|\widehat{\vg}_1 - \vg_1\| \leq \alpha \|\vg_1\| \text{~and~} \ldots \text{~and~} \|\widehat{\vg}_T - \vg_T\| \leq \alpha \|\vg_T\|) \geq \prod_{t=1}^T (1 - \delta / T) \geq 1 - \delta
\end{align*}
Thus, with probability at least $1-\delta$, we have:
\begin{align}\label{eq:app-estimator-condition}
    \|\widehat{\vg}_1 - \vg_1\| \leq \alpha \|\vg_1\| \text{~and~} \ldots \text{~and~} \|\widehat{\vg}_T - \vg_T\| \leq \alpha \|\vg_T\|
\end{align}
A similar condition was also adopted in \cite{homem2008rates} and \cite{chen2018stochastic}. When Equation~(\ref{eq:app-estimator-condition}) is satisfied, we have the following lemma:
\begin{lemma}[Lemma 11 in \cite{chen2018stochastic}]\label{lemma:bound-gradient-estimator}
If $\|\widehat{\vg}_t - \vg_t\| \leq \alpha \|\vg_t\|$, then we have:
$$(1 - \alpha) \|\vg_t\| \leq \|\widehat{\vg}_t\| \leq (1 + \alpha) \|\vg_t\|$$
\end{lemma}

Next we introduce the following property of $M$-smooth function:
\begin{lemma}\label{lemma:m-smooth}
For an $M$-smooth function $\gL(\vtheta)$, we have:
\begin{align*}
    \forall \vtheta_1, \vtheta_2 \in \Theta, ~\gL(\vtheta_2) \leq
    \gL(\vtheta_1) + \langle \nabla \gL(\vtheta_1), \vtheta_2 - \vtheta_1  \rangle + \frac{M}{2} \|\vtheta_2 - \vtheta_1\|^2
\end{align*}
\end{lemma}

From Assumption~\ref{assump:app-smooth} and Lemma~\ref{lemma:m-smooth}, we know that:
\begin{align*}
    \gL(\vtheta_{t+1}) \leq \gL(\vtheta_t) + \langle \nabla \gL(\vtheta_t), \vtheta_{t+1} - \vtheta_t \rangle + \frac{M}{2} \|\vtheta_{t+1} - \vtheta_t\|^2
\end{align*}

From the SGD update rule in Equation~(\ref{eq:app-update-rule}) ($\vtheta_{t+1} = \vtheta_t - \eta_t \widehat{\vg}_t$), Equation~(\ref{eq:app-estimator-condition}) and Lemma~\ref{lemma:bound-gradient-estimator}, we know that:
\begin{align*}
    \gL(\vtheta_{t+1}) &\leq \gL(\vtheta_t) - \eta_t \langle \vg_t, \widehat{\vg}_t \rangle + \frac{M \eta_t^2 \|\widehat{\vg}_t\|^2}{2} \\
    &\leq \gL(\vtheta_t) - \eta_t (1 - \alpha) \|\vg_t\|^2 + \frac{M \eta_t^2 (1 + \alpha)^2 \|\vg_t\|^2}{2}
\end{align*}
Rearranging the above equation, using the gradient norm bound in Equation~(\ref{eq:app-gradient-norm-bound}) and the fact that $\alpha \in (0,1)$, we get:
\begin{align*}
    \left((1-\alpha)\eta_t - \frac{M}{2} \eta_t^2\right) \|\nabla \gL(\vtheta_t)\|^2
    &\leq \gL(\vtheta_t) - \gL(\vtheta_{t+1}) + \left(\alpha M \eta_t^2 + \frac{\alpha^2}{2} M \eta_t^2\right) \|\vg_t\|^2 \\
    &< \gL(\vtheta_t) - \gL(\vtheta_{t+1}) + 6 \alpha M L^2 e^{4 \gamma K} \eta_t^2
\end{align*}
where the condition $\eta_t < 2(1-\alpha)/M$ is used to ensure $(1-\alpha)\eta_t - M \eta_t^2 / 2 > 0$.

Summing up the above inequalities from $t=1$ to $T$, we get:
\begin{align*}
    \sum_{t=1}^T \left(\left((1 - \alpha)\eta_t - \frac{M}{2} \eta_t^2\right) \|\nabla \gL(\vtheta_t)\|^2\right) 
    &<  \sum_{t=1}^T (\gL(\vtheta_t) - \gL(\vtheta_{t+1})) + 6 \alpha M L^2 e^{4 \gamma K} \sum_{t=1}^T \eta_t^2 \\
    &= \gL(\vtheta_1) - \gL(\vtheta_{T}) + 6 \alpha M L^2 e^{4 \gamma K} \sum_{t=1}^T \eta_t^2 \\
    &\leq \gL(\vtheta_1) - \gL^* + 6 \alpha M L^2 e^{4 \gamma K} \sum_{t=1}^T \eta_t^2
\end{align*}
where the last inequality is due to the fact that $\gL^* \leq \gL(\vtheta_{T+1})$.

Dividing both sides by $\sum_{t=1}^T ((1-\alpha)\eta_t - M \eta_t^2/2)$, we get:
\begin{align*}
    \sum_{t=1}^T \left(\frac{2(1-\alpha)\eta_t - M \eta_t^2}{\sum_{t=1}^T (2(1-\alpha)\eta_t - M \eta_t^2)} \|\nabla \gL(\vtheta_t)\|^2\right) < \frac{2(\gL(\vtheta_1) - \gL^*) + 12 \alpha M L^2 e^{4 \gamma K} \sum_{t=1}^T \eta_t^2}{\sum_{t=1}^T (2(1-\alpha)\eta_t - M \eta_t^2)}
\end{align*}

By definition of $p_Z$ in Equation~(\ref{eq:app-def-pz}), which is used to select a final solution among $\{\vtheta_1, \ldots, \vtheta_T\}$, we know that:
\begin{align*}
    \bb{E}_{p_Z}[\|\nabla \gL(\vtheta_Z)\|^2] &= \sum_{t=1}^T (p_Z(t) \|\nabla \gL(\vtheta_t)\|^2) \\
    &= \sum_{t=1}^T \left(\frac{2(1-\alpha)\eta_t - M \eta_t^2}{\sum_{t=1}^T (2(1-\alpha)\eta_t - M \eta_t^2)} \|\nabla \gL(\vtheta_t)\|^2\right) \\
    &< \frac{2(\gL(\vtheta_1) - \gL^*) + 12 \alpha M L^2 e^{4 \gamma K} \sum_{t=1}^T \eta_t^2}{\sum_{t=1}^T (2(1-\alpha)\eta_t - M \eta_t^2)}
\end{align*}
\end{proof}

Now, let us consider a simple case where we use a constant step size, which gives us the following corollary:
\begin{corollary}\label{cor:app-constant-step-size}
Under the conditions in Theorem~\ref{the:app-sgd-convergence} except that we use constant step sizes:
\begin{align}
    \eta_t = \min\left\{ \frac{1 - \alpha}{M}, \frac{1}{\sqrt{T}} \right\},~~t=1,\ldots,T\label{eq:app-constant-step-size}
\end{align}
then with probability at least $1-\delta$, we have:
\begin{align*}
    \bb{E}_{p_Z}[\|\nabla \gL(\vtheta_Z)\|^2] < \frac{2M(\gL(\vtheta_1) - \gL^*)}{(1-\alpha)^2 T} + \frac{2(\gL(\vtheta_1) - \gL^*) + 12 \alpha M L^2 e^{4 \gamma K}}{(1-\alpha)\sqrt{T}}
\end{align*}
\end{corollary}
\begin{proof}
Since we are using a constant step size, by Theorem~\ref{the:app-sgd-convergence}, we know that:
\begin{align*}
    \bb{E}_{p_Z}[\|\nabla \gL(\vtheta_Z)\|^2] < \frac{2(\gL(\vtheta_1) - \gL^*) + 12 \alpha M L^2 e^{4 \gamma K} T \eta_1^2}{T \eta_1 (2(1 - \alpha)-M \eta_1)}
\end{align*}

By Equation~(\ref{eq:app-constant-step-size}), we have:
\begin{align*}
    \frac{2(\gL(\vtheta_1) - \gL^*) + 12 \alpha M L^2 e^{4 \gamma K} T \eta_1^2}{T \eta_1 (2(1 - \alpha)-M \eta_1)}
    &\leq \frac{2(\gL(\vtheta_1) - \gL^*) + 12 \alpha M L^2 e^{4 \gamma K} T \eta_1^2}{T (1-\alpha) \eta_1} \\
    &= \frac{2(\gL(\vtheta_1) - \gL^*)}{T (1-\alpha) \eta_1} +
    \frac{12 \alpha M L^2 e^{4 \gamma K} \eta_1}{1-\alpha} \\
    &\leq \frac{2(\gL(\vtheta_1) - \gL^*)}{T(1-\alpha)} \max\left\{\frac{M}{1-\alpha}, \sqrt{T}\right\} + \frac{12 \alpha M L^2 e^{4 \gamma K}}{(1-\alpha) \sqrt{T}} \\
    & < \frac{2M(\gL(\vtheta_1) - \gL^*)}{(1-\alpha)^2 T} + \frac{2(\gL(\vtheta_1) - \gL^*) + 12 \alpha M L^2 e^{4 \gamma K}}{(1-\alpha)\sqrt{T}}
\end{align*}
\end{proof}

Note that an alternative result for constant step sizes can also be obtained from Theorem 6 in \cite{chen2018stochastic}:

\begin{theorem}\label{the:app-sgd-convergence-alternative}
Under Assumptions~\ref{assump:app-energy} and \ref{assump:app-smooth}, for arbitrary constants $\alpha \in (0,1)$ and $\delta \in (0,1)$, suppose we use constant step sizes:
\begin{align}
    \eta_t = \frac{\sqrt{2(\gL(\vtheta_1) - \gL^*)}}{(1+\alpha) 2L e^{2 \gamma K} \sqrt{MT}}
\end{align}
and the sample size $N_t$ used for estimating $\widehat{\vg}_t$ satisfies Equation~(\ref{eq:app-convergence-sample}), then with probability at least $1-\delta$, we have:
\begin{align}
    \min_{t=1,\ldots,T} \|\nabla \gL(\vtheta_t)\|^2 \leq \frac{(1+\alpha) 2L e^{2\gamma K} \sqrt{2M(\gL(\vtheta_1) - \gL^*)}}{(1-\alpha)\sqrt{T}}
\end{align}
\end{theorem}

Although both Corollary~\ref{cor:app-constant-step-size} and Theorem~\ref{the:app-sgd-convergence-alternative} give a convergence rate of $O(1/\sqrt{T})$, the strategy in Theorem~\ref{the:app-sgd-convergence-alternative} requires extra computational effort to compute $\|\nabla L(\vtheta_t)\|$ for $t=1,\ldots,T$ in order to select the solution with the minimum gradient norm. Since $\|\nabla L(\vtheta)\|$ cannot be computed exactly, Monte Carlo estimation will incur additional approximation error. By contrast, the strategy in our analysis does not have such issues and Theorem~\ref{the:app-sgd-convergence} provides a general analysis on the convergence rate of the randomized SGD algorithm with consistent but biased gradient estimators, which allows using different step sizes. 

For example, starting from Equation~(\ref{eq:sgd-converge-gradient-bound}) and with the fact that:
\begin{align*}
    \sum_{t=1}^T t = O(T^2),~~\sum_{t=1}^T \sqrt{t} = O(T^{\frac{3}{2}}),~~\sum_{t=1}^T t^{-\frac{1}{4}} = O(T^{\frac{3}{4}}),~~\sum_{t=1}^T t^{-\frac{1}{2}} = O(T^{\frac{1}{2}})
\end{align*}
one can easily verify that using the following increasing step sizes:
\begin{align*}
    \eta_t = \min\left\{ \frac{1 - \alpha}{M}, \frac{\sqrt{t}}{T} \right\},~~t=1,\ldots,T
\end{align*}
or decreasing step sizes:
\begin{align*}
    \eta_t = \min\left\{ \frac{1 - \alpha}{M}, \frac{1}{(tT)^{1/4}} \right\},~~t=1,\ldots,T
\end{align*}
will give us a similar convergence rate of $O(1/\sqrt{T})$.

Finally, if we would like to make a stronger assumption that the loss function $\gL(\vtheta)$ is strongly convex, then we can obtain a stronger result that $\vtheta_T$ converges to the optimal solution $\vtheta^*$ in $L_2$-norm with a convergence rate of $O(1/T)$.

\begin{assumption}\label{assump:app-strong-convex}
The loss function $\gL(\vtheta)$ is $J$-strongly convex (with $J>0$):
\begin{align*}
    \forall \vtheta_1,\vtheta_2 \in \Theta, 
    \gL(\vtheta_2) - \gL(\vtheta_1) \geq \langle \nabla \gL(\vtheta_1), \vtheta_2 - \vtheta_1 \rangle + \frac{J}{2} \|\vtheta_2 - \vtheta_1\|^2.
\end{align*}
and the unique optimum of $\gL(\vtheta)$ is $\vtheta^*$.
\end{assumption}

\begin{theorem}\label{the:app-sgd-convergence-strong-convex}
Under Assumptions~\ref{assump:app-energy} and \ref{assump:app-strong-convex}, for arbitrary constant $\delta \in (0,1)$, suppose that $J \leq 2L e^{2 \gamma K}/\|\vtheta_1 - \vtheta^*\|$ and we use decreasing step sizes:
\begin{align}
    \eta_t = \frac{1}{(J - J/(2T))t},~~t=1,\ldots,T
\end{align}
and the sample size $N_t$ used for estimating $\widehat{\vg}_t$ satisfies:
\begin{align}
    N_t \geq \frac{128 L^2 T^2 e^{8 \gamma K} (1 + 4 \log(2T/\delta))}{J^2 \|\vg_t\|^2}
\end{align}
then with probability at least $1-\delta$, we have:
\begin{align}
    \|\vtheta_T - \vtheta^*\| \leq \frac{4 L^2 e^{4 \gamma K}}{T} 
    \left[ \frac{(2+J/T)^2 + J^2 (2 - 1/T)}{J^2 (2 - 1/T)^2} \right]
    \label{eq:app-sgd-convergence-convex}
\end{align}
\end{theorem}
Intuitively, Theorem~\ref{the:app-sgd-convergence-strong-convex} implies that when the loss function $\gL(\vtheta)$ is strongly convex in $\vtheta$ with $\vtheta^*$ being the optimal solution, then under some conditions on the sample sizes for estimating the gradients and step sizes for updating the parameters, the output of the SGD algorithm $\vtheta_T$ will converge to $\vtheta^*$ with a convergence rate of $O(1/T)$.

When the loss function $\gL(\vtheta)$ is convex but not strongly convex, we have the following theorem showing a typical convergence rate of $O(1/\sqrt{T})$:
\begin{assumption}\label{assump:app-convex}
The loss function $\gL(\vtheta)$ is convex and the parameter space has finite diameter $D$:
$\sup_{\vtheta_1, \vtheta_2 \in \Theta} \|\vtheta_1\| = D$.
Let $\vtheta^* \in \argmin_{\vtheta \in \Theta} \gL(\vtheta)$.
\end{assumption}
\begin{theorem}\label{the:app-sgd-convergence-convex}
Under Assumptions~\ref{assump:app-energy} and \ref{assump:app-convex}, for arbitrary constant $\delta \in (0,1)$, suppose we use decreasing step sizes $\eta_t = 1/\sqrt{t}$ for $t=1,\ldots,T$ and the sample size $N_t$ used for estimating $\widehat{\vg}_t$ satisfies:
\begin{align}
    N_t \geq \frac{32 L^2 T e^{8 \gamma K} (1 + 4 \log(2T/\delta))}{ \|\vg_t\|^2}
\end{align}
then with probability at least $1-\delta$, we have:
\begin{align*}
    f(\overline{\vtheta}_T) - f(\vtheta^*) \leq \frac{1}{\sqrt{T}} \left[ D^2 + 2 L^2 e^{4 \gamma K}  \left(1 + \left(1 + \frac{1}{\sqrt{T}}\right)^2 \sqrt{1 + \frac{1}{T}}\right) \right]
\end{align*}
where $\overline{\vtheta}_T \defeq \frac{1}{T}\sum_{t=1}^T \vtheta_t$.
\end{theorem}

The above theorems follow from the sample complexity bound in Theorem~\ref{the:app-sample-complexity} and the results in \cite{chen2018stochastic} (Theorem~2 with constant $\rho = J/2$ and Theorem~5 with constants $\rho=1, c=1$), which we refer to for a detailed proof. Note that the condition on $J$ in Theorem~\ref{the:app-sgd-convergence-strong-convex} is optional and without the condition, we can obtain the same convergence rate at the cost of a cumbersome form in the R.H.S. of Equation~(\ref{eq:app-sgd-convergence-convex}).

\newpage
\section{Additional Experimental Details}\label{app:experiments}
\subsection{2-D Synthetic Data Experiments}\label{app:synthetic-data}

\begin{figure*}[h!]
    \centering
\begin{subfigure}{0.8\textwidth}
    \centering
    \includegraphics[width=\textwidth]{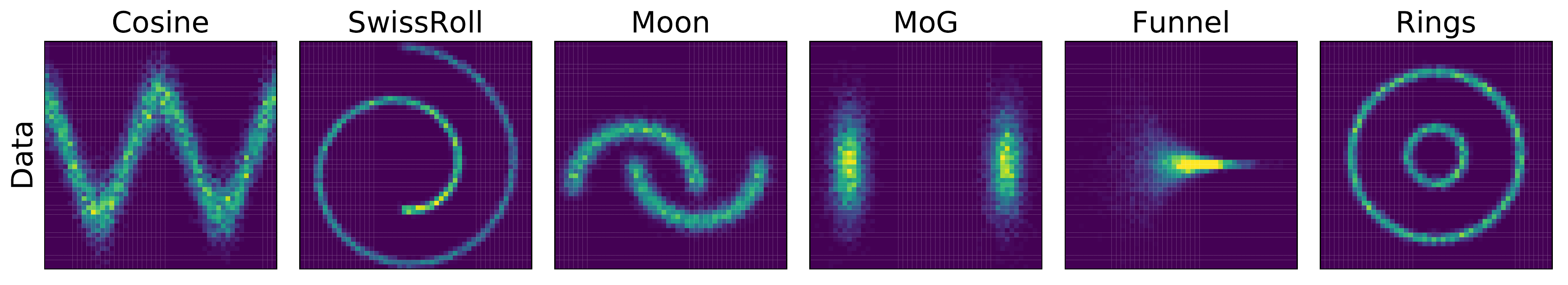}
\end{subfigure}
\begin{subfigure}{0.8\textwidth}
    \centering
    \includegraphics[width=\textwidth]{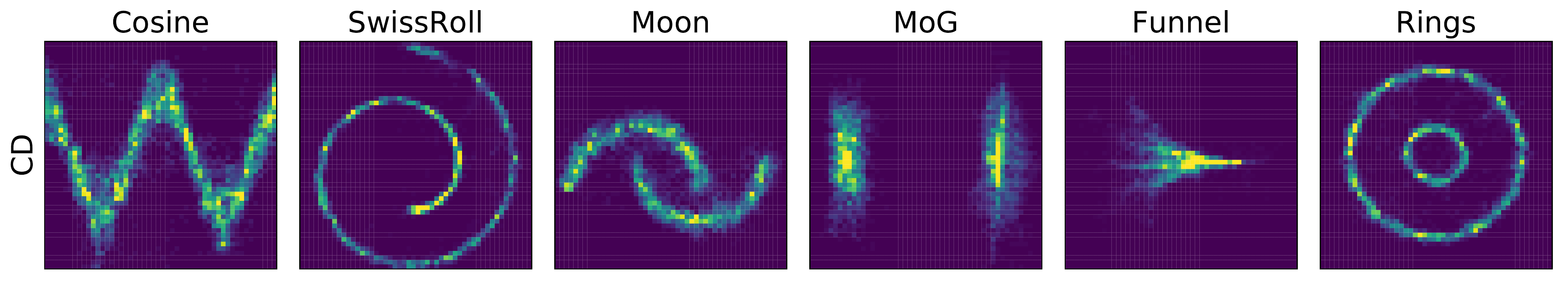}
\end{subfigure}
\begin{subfigure}{0.8\textwidth}
    \centering
    \includegraphics[width=\textwidth]{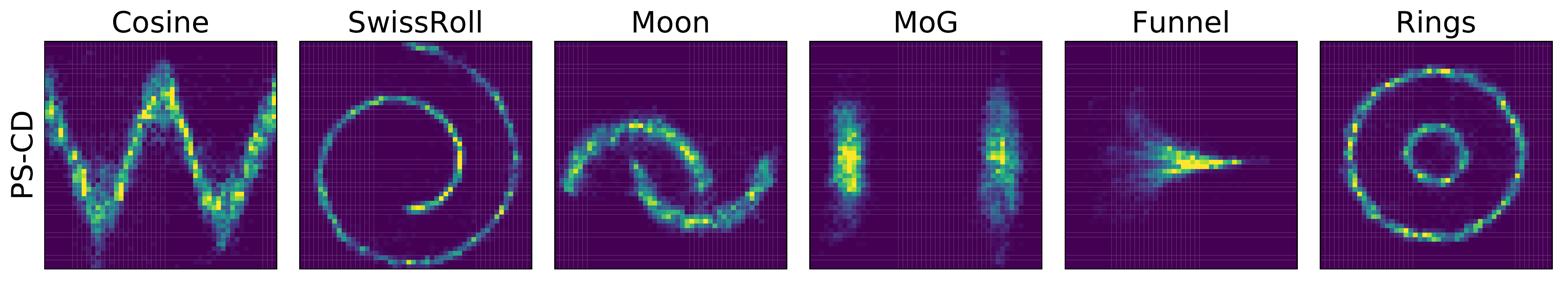}
\end{subfigure}
    \caption{Histograms of samples from the data distribution (top), CD (middle) and PS-CD (bottom).}
    \label{fig:hist2d}
\end{figure*}

\begin{table*}[h!]
\centering
\caption{Maximum mean discrepancy (MMD, multiplied by $10^4$) results on six 2-D synthetic datasets. Lower is better. CD denotes contrastive divergence algorithm, and PS-CD denotes the pseudo-spherical contrastive divergence algorithm (with $\gamma=1.0$).
}
\label{tab:synthetic}
\scalebox{0.95}{
\begin{tabular}{cccccccc}
\toprule
Method & Cosine & Swiss Roll & Moon & MoG & Funnel & Rings\\
\midrule
CD  & $1.20 \pm 0.45$  & $3.39 \pm 0.48$  & $0.64 \pm 0.11$  & $3.01 \pm 0.58$  & $\textbf{1.56} \pm 0.65$  & $2.79 \pm 0.63$ \\
PS-CD & $\textbf{0.86} \pm 0.12$  & $\textbf{0.89} \pm 0.39$  & $\textbf{0.12} \pm 0.04$  & $\textbf{1.78} \pm 0.35$  & $2.34 \pm 0.45$  & $\textbf{2.02} \pm 0.32$ \\\bottomrule
\end{tabular}
}
\end{table*}

\subsection{Understanding the Effects of Different $\gamma$ Values in 1-D Examples}\label{app:visualize-objective-landscape}

\begin{figure}[t!]
    \begin{subfigure}{.33\textwidth}
      \centering
      \includegraphics[width=\textwidth]{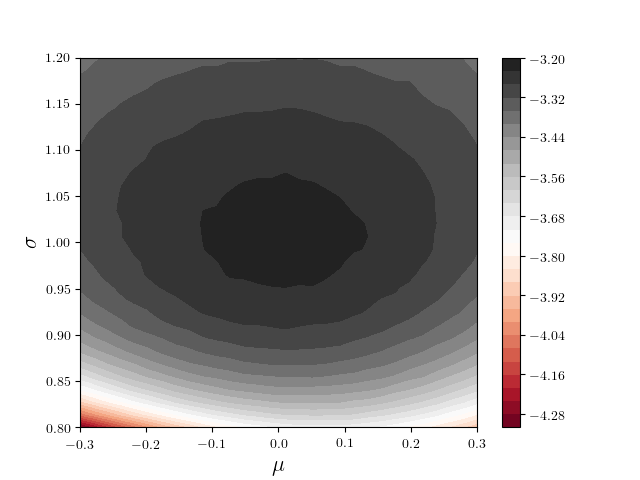} 
      \caption{$\gamma=-0.5$}
    \end{subfigure}
    \begin{subfigure}{.33\textwidth}
      \centering
      \includegraphics[width=\textwidth]{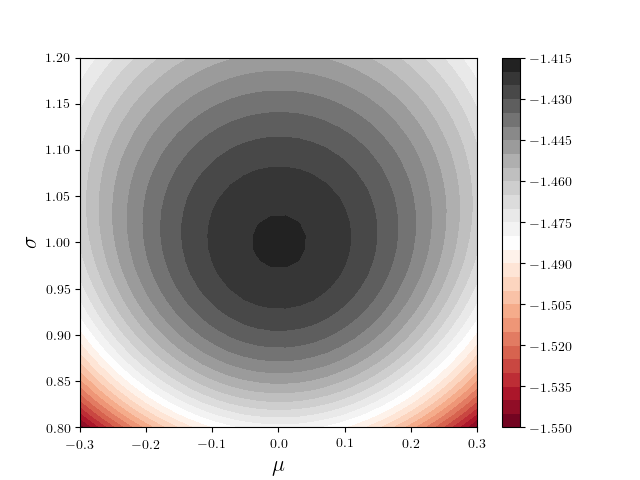}
      \caption{$\gamma=0$}
    \end{subfigure}
    \begin{subfigure}{.33\textwidth}
      \centering
      \includegraphics[width=\textwidth]{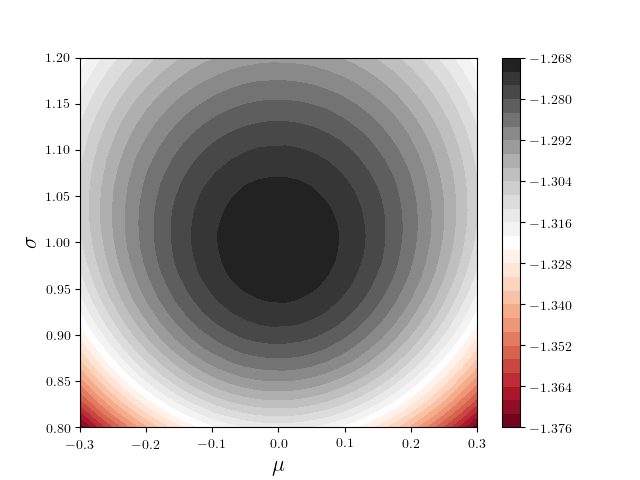}
      \caption{$\gamma=0.1$}
    \end{subfigure}
    \begin{subfigure}{.33\textwidth}
      \centering
      \includegraphics[width=\textwidth]{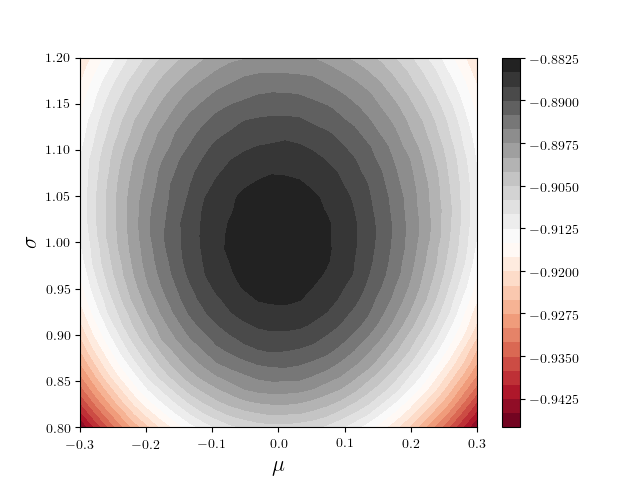}
      \caption{$\gamma=0.5$}
    \end{subfigure}
    \begin{subfigure}{.33\textwidth}
      \centering
      \includegraphics[width=\textwidth]{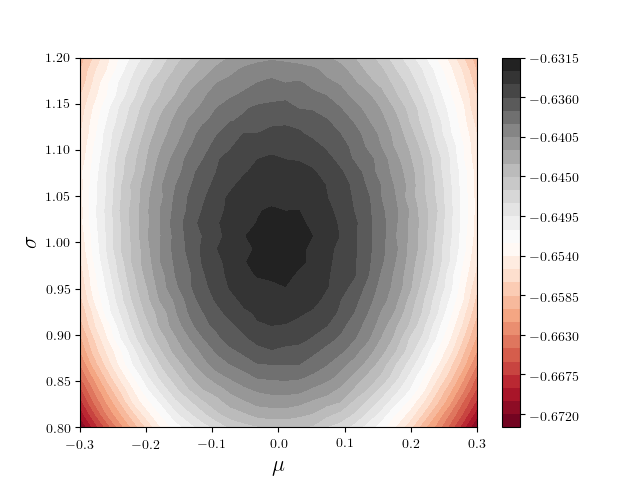}
      \caption{$\gamma=1.0$}
    \end{subfigure}
    \begin{subfigure}{.33\textwidth}
      \centering
      \includegraphics[width=\textwidth]{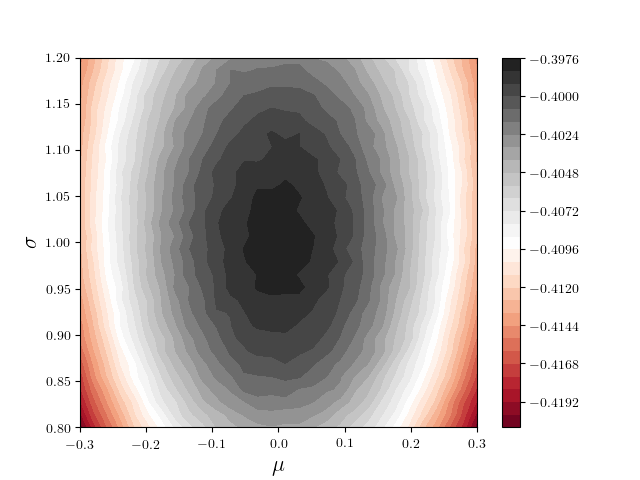}
      \caption{$\gamma=2.0$}
    \end{subfigure}
	\caption{Visualization of different objective landscapes for model well-specified scenarios. $\gamma=0$ corresponds to the logarithm scoring rule (MLE) and other values correspond to the $\gamma$-scoring rules.}
	\label{fig:well-specified}
\end{figure}

\begin{figure}[t!]
    \begin{subfigure}{.33\textwidth}
      \centering
      \includegraphics[width=\textwidth]{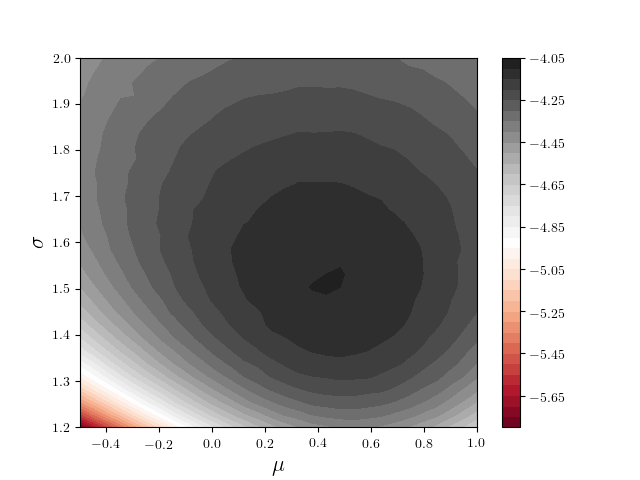} 
      \caption{$\gamma=-0.5$}
    \end{subfigure}
    \begin{subfigure}{.33\textwidth}
      \centering
      \includegraphics[width=\textwidth]{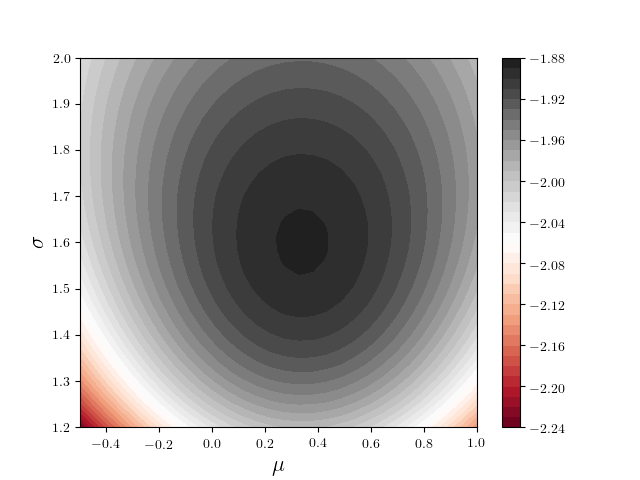}
      \caption{$\gamma=0$}
    \end{subfigure}
    \begin{subfigure}{.33\textwidth}
      \centering
      \includegraphics[width=\textwidth]{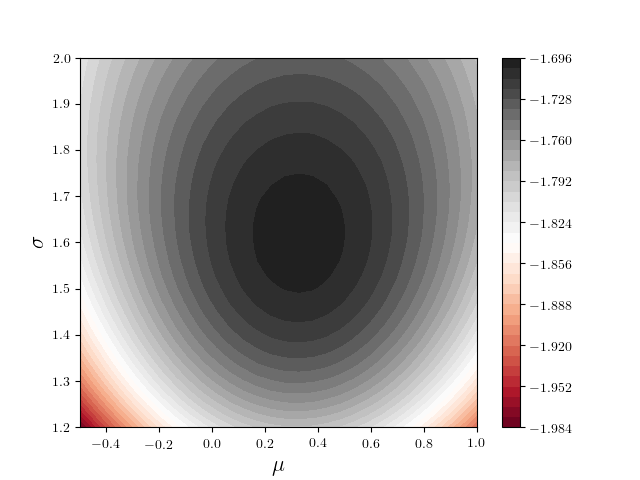}
      \caption{$\gamma=0.1$}
    \end{subfigure}
    \begin{subfigure}{.33\textwidth}
      \centering
      \includegraphics[width=\textwidth]{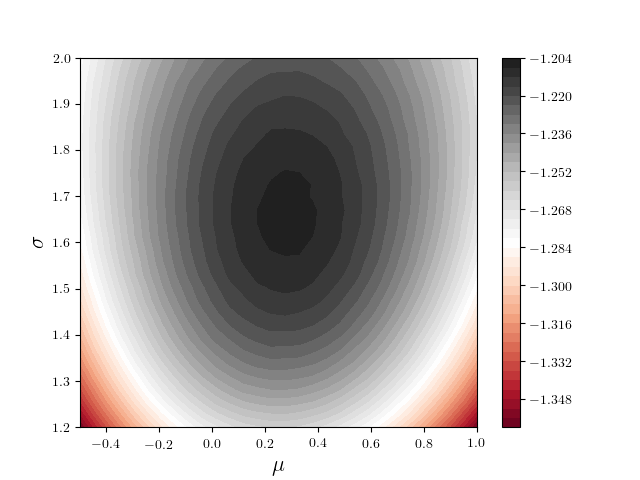}
      \caption{$\gamma=0.5$}
    \end{subfigure}
    \begin{subfigure}{.33\textwidth}
      \centering
      \includegraphics[width=\textwidth]{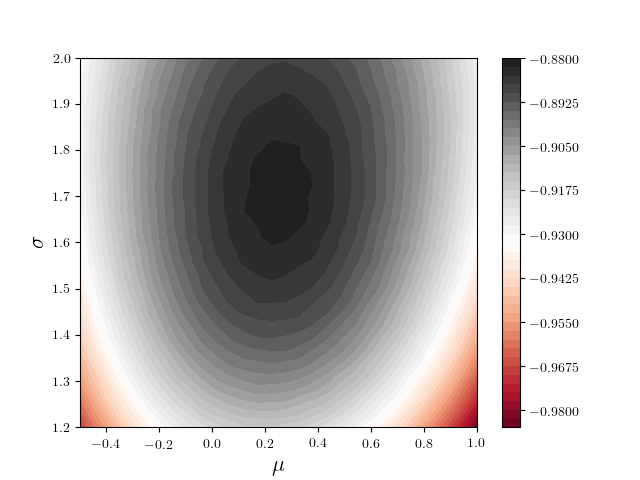}
      \caption{$\gamma=1.0$}
    \end{subfigure}
    \begin{subfigure}{.33\textwidth}
      \centering
      \includegraphics[width=\textwidth]{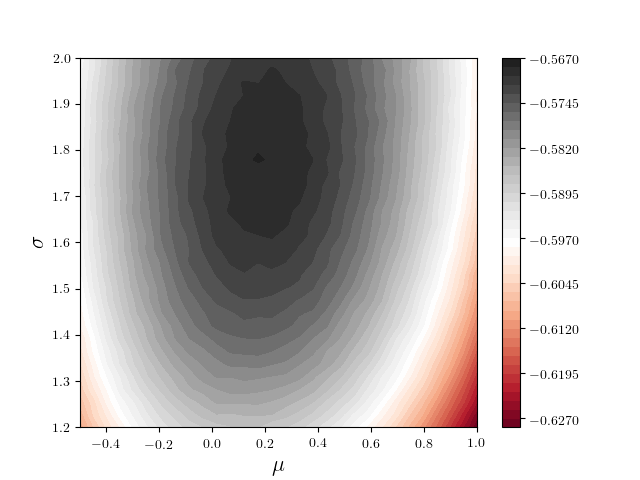}
      \caption{$\gamma=2.0$}
    \end{subfigure}
	\caption{Visualization of different objective landscapes for model mis-specified scenarios. $\gamma=0$ corresponds to the logarithm scoring rule (MLE) and other values correspond to the $\gamma$-scoring rules.}
	\label{fig:mis-specified}
\end{figure}

In this section, we aim to provide insights on the effects of different $\gamma$ values with 1-D toy experiments. Specifically, we use an EBM with a quadratic energy function (corresponding to a Gaussian distribution):
\begin{align}
    E_{\mu, \sigma}(x) = \frac{(x - \mu)^2}{2\sigma^2},~q_{\mu, \sigma}(x) \propto \exp(-E_{\mu, \sigma}(x))\label{eq:gaussian}
\end{align}
where $\mu$ and $\sigma$ are two trainable parameters.

First, we show that when the real data distribution is also a Gaussian distribution such that the model is well-specified, then different $\gamma$ values will induce the same optimal distribution since they are \emph{strictly proper}. To verify this property, we visualize the objective landscape in Figure~\ref{fig:well-specified}.

Second, we use the same quadratic energy function to fit a mixture of Gaussians. We visualize the objective landscapes in Figure~\ref{fig:mis-specified}, which shows that when the model is mis-specified, different objectives will exhibit different modeling preferences (inducing different solutions). This corresponds to the practical scenarios, where such property enables us to flexibly specify different inductive biases to make tradeoff among various modeling factors such as diversity/quality.

\newpage
\subsection{Understanding the Effects of Different $\gamma$ Values in Image Generation}\label{app:effects-of-gamma-image}
Although FID has been the most popular evaluation metric for image generative models, it is problematic since it summarizes the difference between two distributions into a single number and fails to separate important aspects such as fidelity and diversity \cite{naeem2020reliable}. To better demonstrate the modeling flexibility brought by the proposed PS-CD framework, we conduct experiments on CIFAR-10 dataset using a set of more indicative and reliable metrics proposed by \cite{naeem2020reliable} to evaluate the effects of $\gamma$ from various perspectives.

\begin{table}[t]
    \centering
    \caption{Effects of $\gamma$ on CIFAR-10 image generation. 
    We use the same image embeddings (activations of a pre-trained inception network) to compute these metrics and FID to ensure consistency. 
    We briefly introduce these metrics here and refer to \cite{naeem2020reliable} for accurate descriptions and mathematical definitions. Denote data distribution as $P(X)$ and model distribution as $Q(X)$. Based on manifold estimation, \emph{Precision} is defined as the portion of $Q(X)$ that can be generated by $P(X)$ and \emph{Recall} is symmetrically defined as the portion of $P(X)$ that can be generated by $Q(X)$; \emph{Density} improves upon \emph{Precision} to count how many real-sample neighbourhood
    spheres contain a certain fake sample; \emph{Coverage} improves upon \emph{Recall} to measure the fraction of real samples whose neighbourhoods contain at least one fake sample.}
    \label{tab:cifar}
    \vspace{10pt}
    \begin{tabular}{l c c c c c}
    \toprule
       & Density & Coverage & Precision & Recall & FID\\
      \midrule
      CD ($\gamma=0$)  & 0.693 & 0.601 & 0.798 & 0.368 & 37.90 \\
      PS-CD ($\gamma=-0.5$) & 0.906 & 0.691 & 0.848 & 0.360 & 27.95 \\
      PS-CD ($\gamma=0.5$) & 0.772 & 0.634 & 0.819 & 0.352 & 35.02\\
      PS-CD ($\gamma=1.0$) & 0.929 & 0.694 & 0.853 & 0.341 & 29.78 \\
      PS-CD ($\gamma=2.0$) & 0.932 & 0.652 & 0.861 & 0.351 & 33.19\\
      \bottomrule
    \end{tabular}
    \label{tab:effects-of-gamma-cifar10}
\end{table}

From Table~\ref{tab:cifar}, we have some interesting observations: (1) PS-CD with $\gamma=-0.5$ and $\gamma=1.0$ get best FID scores because they can simultaneously achieve good balance among these metrics (e.g., high \emph{Density} and \emph{Coverage} score); (2) By contrast, PS-CD with $\gamma=2.0$ achieves the highest \emph{Density} score but a relatively low \emph{Coverage} score, which potentially leads to a slightly worse FID; (3) Many members in the PS-CD family showed superior performance over traditional contrastive divergence in most metrics, demonstrating the potential of our method. Just like various $f$-divergences used in generative modeling, different members in the PS-CD family can represent complicated inductive bias in practice (although being strictly proper in model well-specified case). Since these single-valued evaluation metrics measure the generative performance in a complicated way, we think it is normal that the change of $\gamma$ is not monotone to the change of each metric. For specific application scenarios, we may mainly care about a certain metric and we should choose $\gamma$ accordingly. 

We would like to emphasize that, a major contribution of our paper is opening the door to a new family of EBM training objectives and enabling us to flexibly specify modeling preferences, without introducing additional computational cost compared to CD (unlike adversarial training in $f$-EBM).

\subsection{Image Generation Samples for PS-CD}\label{app:image-generation-samples}
\begin{figure*}[h]
	\begin{center}
		\includegraphics[width=.3\textwidth]{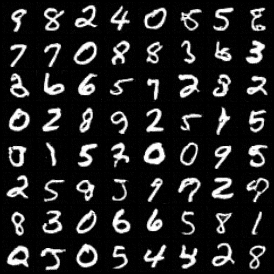}
		\hspace{3pt}
		\includegraphics[width=.3\textwidth]{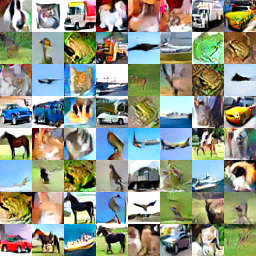}
		\hspace{3pt}
		\includegraphics[width=.3\textwidth]{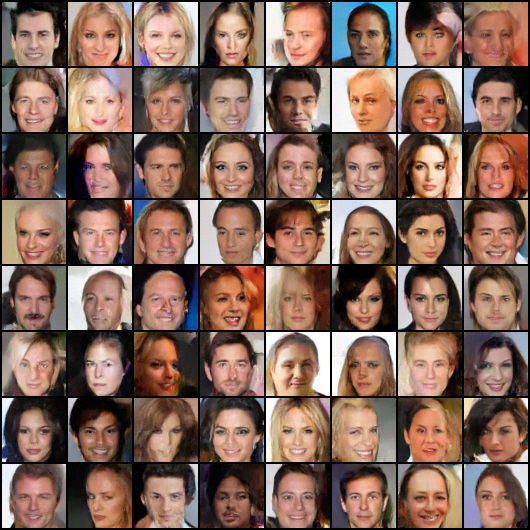}
	\end{center}
	\caption{MNIST, CIFAR-10 and CelebA samples for PS-CD ($\gamma=1.0$).}
	\label{fig:samples}
\end{figure*}

\newpage
\subsection{Training Details}

\begin{wrapfigure}{r}{4.5cm}
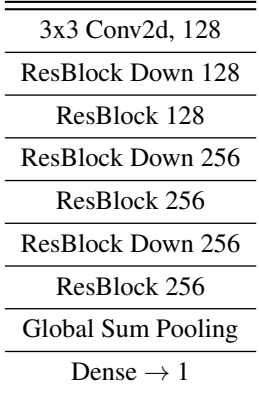
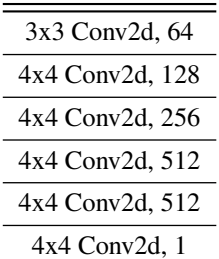

\begin{subfigure}[t]{0.3\textwidth}
\centering
\begin{tabular}{c}
    \toprule
    \toprule
    3x3 Conv2d, 128 \\
    \midrule
    ResBlock Down 128 \\
    \midrule
    ResBlock 128 \\
    \midrule
    ResBlock Down 256 \\
    \midrule
    ResBlock 256 \\ 
    \midrule
    ResBlock Down 256 \\
    \midrule
    ResBlock 256 \\
    \midrule
    Global Sum Pooling \\ 
    \midrule
    Dense $\rightarrow$ 1\\
    \bottomrule
\end{tabular}
\caption{CIFAR-10 ($32 \times 32$)}
\vspace{10pt}
\end{subfigure}
\\
\begin{subfigure}[t]{0.3\textwidth}
\centering
\begin{tabular}{c}
    \toprule
    \toprule
    3x3 Conv2d, 64 \\
    \midrule
    4x4 Conv2d, 128 \\
    \midrule
    4x4 Conv2d, 256 \\
    \midrule
    4x4 Conv2d, 512 \\
    \midrule
    4x4 Conv2d, 512 \\
    \midrule
    4x4 Conv2d, 1 \\
    \bottomrule
\end{tabular}
\caption{CelebA ($64 \times 64$)}
\end{subfigure}
\caption{Network architectures.}\label{fig:architectures}
\vspace{-10pt}
\end{wrapfigure}

To keep a fair comparison, all the compared methods use the same architecture to implement the energy function, except that $f$-EBMs require an additional variational function that uses the same architecture as the energy function. The architectures used for CIFAR-10 ($32 \times 32$) and CelebA ($64 \times 64$) datasets are shown in Figure~\ref{fig:architectures}. We use leaky-ReLU non-linearity with default leaky factor $0.2$ throughout the architectures (between all the convolution layers). Following \cite{du2019implicit, yu2020training}, we apply spectral normalization/$L_2$ regularization (on the outputs of the models) with coefficient 1.0 to improve the stability.

For CIFAR-10, to keep a fair comparison, we use the same sampling strategy as \cite{du2019implicit,yu2020training}, where a sample replay buffer is employed to improve the mixing of Langevin dynamics. Specifically, we use 60 steps Langevin dynamics together with a sample replay buffer of size 10000 to produce samples in the training phase. In each Langevin step, we use a step size of 10.0 and a random noise with standard deviation of 0.005.

For CelebA, which has a higher data dimension, we use the sampling strategy in \cite{nijkamp2019learning} to improve the efficiency of sampling, where we always start the Markov chains from a fixed uniform distribution and run a fixed number of Langevin steps ($100$) with a constant step size.

For all the experiments, we use Adam optimizer to optimize the parameters of the energy function. In each training iteration, we use a batch size of $128$ for CIFAR-10 and $64$ for CelebA. We run the PS-CD algorithms for about 50K iterations of parameter updates for CIFAR-10 and about 100K iterations for CelebA.

For computational cost, the CIFAR-10 experiments take about 48 hours
on 4 Titan Xp GPUs, while the CelebA experiments take about 16 hours since we learn non-convergent short-run MCMC.

\subsection{OOD Detection \& Robustness to Data Contamination}\label{app:ood}
\begin{table}[h]
    \centering
    \caption{OOD Detection results (AUROC score) for models trained on CIFAR-10.}
    \label{tab:ood}
    \vspace{5pt}
    \begin{tabular}{l c c c c}
    \toprule
      OOD Dataset & PixelCNN++ & Glow & CD & PS-CD\\
      \midrule
      SVHN  & 0.32 & 0.24 & 0.43 & 0.56 \\
      Textures & 0.33 & 0.27 & 0.36 & 0.44 \\
      Uniform/Gaussian & 1 & 1 & 1 & 1\\
      CIFAR-10 Interpolation & 0.71 & 0.59 & 0.63 & 0.68 \\
      CelebA & / & / & 0.51 & 0.58\\
      \bottomrule
    \end{tabular}
\end{table}

\begin{table}[h]
    \centering
    \caption{Training EBMs under data contamination on CIFAR-10. We measure the change of FID score after training with the contaminated dataset.}
    \label{tab:contamination-cifar}
    \vspace{5pt}
    \resizebox{\textwidth}{!}{
    \begin{tabular}{c c c c c c}
    \toprule
       & Pretrained Model & CD 1000 Steps & CD 2000 Steps & PS-CD 1000 Steps & PS-CD 2000 Steps\\
      \midrule
       FID & 68.77 & 95.56 & 300.89 & 59.78 & 57.24 \\
      \bottomrule
    \end{tabular}}
\end{table}

\begin{figure*}[h]
	\begin{center}
		\includegraphics[width=\textwidth]{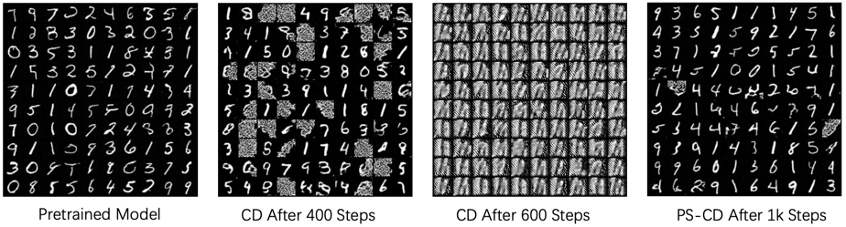}
		\includegraphics[width=\textwidth]{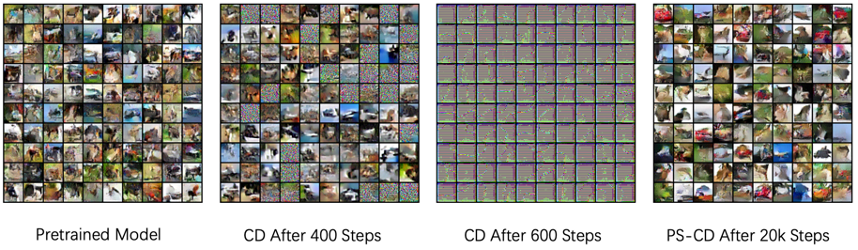}
	\end{center}
	\caption{Samples after training with the contaminated dataset on MNIST and CIFAR-10.}
	\label{fig:samples-contamination}
\end{figure*}

To show the practical advantage of PS-CD in face of data contamination, we further conduct experiments on MNIST and CIFAR-10 datasets, where we use random uniform noise as the contamination distribution and the contamination ratio is 0.1 (i.e. 10\% images in the training set are replaced with random noise). After a warm-up pretraining (when the model has some OOD detection ability), we train the model with the contaminated data and measure the training progress of CD and PS-CD. 

As shown in Figure~\ref{fig:samples-contamination}, CD gradually generates more and more random noise and diverge after a few training steps, while PS-CD is very robust. In particular, as shown in Table~\ref{tab:contamination-cifar}, for a slightly pretrained unconditional CIFAR-10 model (a simple 5-layer CNN with FID of 68.77), we observe that the performance of CD degrades drastically in terms of FID, while PS-CD can continuously improve the model even using the contaminated data.

We believe that robustness to data contamination is a valuable property for modern deep generative models and there is actually a natural interpretation for the robustness of PS-CD. Compared to CD, there is an extra weight term before the gradient of the energy: 
$\frac{\exp(-\gamma E_\vtheta(\vx_i))}{\sum_j \exp(-\gamma E_\vtheta(\vx_j)} \nabla_\vtheta E_\vtheta(\vx_i)$
(the first term in Eq.~(\ref{eq:gradient-estimator})). Suppose $\vx_i \sim \omega$
is a noise data from the contaminated distribution $\tilde{p}$ in a batch of samples, for a model with OOD detection ability, it will assign a much higher energy to $\vx_i$
than normal data and the weight before $\nabla_\vtheta E_\vtheta(\vx_i)$ will be close to zero. In short, PS-CD naturally integrates the OOD detection ability of EBMs into the training process, which then leads to robustness to data contamination.

\end{document}